%% file: main_arxiv.tex
\documentclass{article}
\input{mathdefs}

\usepackage{graphicx}
\usepackage[labelformat=simple]{subcaption}

\newtheorem{prop}{Proposition}

\newtheorem{thm}{Theorem}

\newtheorem{defn}{Definition}
\renewcommand\thesubfigure{(\alph{subfigure})}

\usepackage[accepted]{icml2022}
\usepackage{etoolbox}
\usepackage{xcolor}
\usepackage{colour-blind}
\usepackage[draft]{minted}

\DeclareFixedFont{\ttb}{T1}{txtt}{bx}{n}{12} 
\DeclareFixedFont{\ttm}{T1}{txtt}{m}{n}{12}  

\definecolor{deepblue}{rgb}{0,0,0.5}
\definecolor{deepred}{rgb}{0.6,0,0}
\definecolor{deepgreen}{rgb}{0,0.5,0}

\usepackage{listings}

\newcommand\pythonstyle{\lstset{
language=Python,
basicstyle=\ttm,
morekeywords={self},              
keywordstyle=\ttb\color{deepblue},
emph={MyClass,__init__},          
emphstyle=\ttb\color{deepred},    
stringstyle=\color{deepgreen},
frame=tb,                         
showstringspaces=false
}}

\lstnewenvironment{python}[1][]
{
\pythonstyle
\lstset{#1}
}
{}

\newcommand\pythoninline[1]{{\pythonstyle\lstinline!#1!}}

\icmltitlerunning{Safety Augmented (Saute) RL}

\begin{document}

\twocolumn[
\icmltitle{Saut\'e RL: Almost Surely Safe Reinforcement Learning \\
Using State Augmentation}



\icmlsetsymbol{equal}{*}

\begin{icmlauthorlist}
\icmlauthor{Aivar Sootla}{huawei}
\icmlauthor{Alexander I. Cowen-Rivers}{huawei,tud}
\icmlauthor{Taher Jafferjee}{huawei}
\icmlauthor{Ziyan Wang}{huawei}
\icmlauthor{David Mguni}{huawei}
\icmlauthor{Jun Wang}{ucl}
\icmlauthor{Haitham Bou-Ammar}{huawei,hucl}
\end{icmlauthorlist}

\icmlaffiliation{tud}{Technische Universit\"at Darmstadt.}
\icmlaffiliation{huawei}{Huawei R\&D UK.}
\icmlaffiliation{ucl}{University College London}
\icmlaffiliation{hucl}{Honorary Lecturer at University College London}

\icmlcorrespondingauthor{Aivar Sootla}{aivar.sootla@huawei.com}
\icmlcorrespondingauthor{Haitham Bou-Ammar}{haitham.ammar@huawei.com}
\icmlcorrespondingauthor{Jun Wang}{jun.wang@cs.ucl.ac.uk}

\icmlkeywords{Reinforcement Learning, Lagrangian Method, State Augmentation}

\vskip 0.3in
]



\printAffiliationsAndNotice{}  

\begin{abstract}
Satisfying safety constraints almost surely (or with probability one) can be critical for the deployment of Reinforcement Learning (RL) in real-life applications. For example, plane landing and take-off should ideally occur with probability one. We address the problem by introducing Safety Augmented (Saut\'e) Markov Decision Processes (MDPs), where the safety constraints are eliminated by augmenting them into the state-space and reshaping the objective. We show that Saut\'e MDP satisfies the Bellman equation and moves us closer to solving Safe RL with constraints satisfied almost surely. We argue that Saut\'e MDP allows viewing the Safe RL problem from a different perspective enabling new features. For instance, our approach has a plug-and-play nature, i.e., any RL algorithm can be ``Saut\'eed''. Additionally, state augmentation allows for policy generalization across safety constraints. We finally show that Saut\'e RL algorithms can outperform their state-of-the-art counterparts when constraint satisfaction is of high importance.
\end{abstract}


\section{Introduction}
\label{introduction}
Reinforcement Learning (RL) offers a great framework for solving sequential decision-making problems using interactions with an environment~\cite{sutton2018reinforcement}. In this context, safety constraint satisfaction and robustness are vital properties for the successful deployment of RL algorithms. Safe RL has received significant attention in recent years, but many unsolved challenges remain, e.g., constraint satisfaction during and after training, efficient algorithms, etc. While different types of constraints were considered in the past, e.g. averaged over trajectory, CVaR, and chance constraints~\cite{chow2017risk}, satisfying constraints almost surely (or with probability one) received less attention to date. This problem is quite important as many safety-critical applications require constraint satisfaction almost surely. For example, we ideally would like to guarantee a safe plane landing and take-off with probability one, while landing with a probability of $0.99$ may not be sufficient. Similarly, an autonomous vehicle should be able to stay in the lane with probability one, while keeping this constraint ``on average'' can potentially lead to catastrophic consequences. 

We represent the safety constraints as a (discounted) sum of nonnegative costs bounded from above by (what we call) \emph{a safety budget}. As the safety costs are accumulated during an episode, there is less freedom in choosing a safe trajectory and hence the available safety budget diminishes over time. We can treat the remaining safety budget as a new state quantifying the risk of constraint violation. This idea can be traced back to classical control methods of augmenting the safety constraints into the state-space~cf.~\cite{daryin2005nonlinear}, however, we adapt it to stochastic problems and RL applications. First, we reshape the objective to take into account the remaining safety budget by assigning infinite values if the budget was spent. Thus we obtain Safety AUgmenTEd (Saut\'e) Markov Decision Process (MDP). Incorporating the constraint into the objective enforces this constraint for every controlled trajectory which leads to constraint satisfaction with probability one. Furthermore, Saut\'e MDP satisfies the Bellman equation justifying the use of \emph{any critic-based RL method}. Finally, since the value functions now additionally depend on the safety information, the cost-to-go in our setting implicitly includes information about possible safety constraint violations. 

Employing the state augmentation for safe RL has been explored in the past. For example,~\cite{calvo2021state} augmented the Lagrangian multiplier into the state space, while keeping the Lagrangian objective. The Lagrangian multiplier contains information about the safety cost that could be exploited by the policy. We argue that such a construction is not required as we can track the accumulated safety cost instead. \cite{chow2017risk} augmented the CVaR constraint in the state-space, however, their augmentation had a realization suitable only for the CVaR constraint. They also had to resort to the Lagrangian approach for solving the problem due to (again) the specificity of the CVaR constraint.  We discuss in detail these state augmentation methods in Appendix~\ref{app:state_augmentation}, but note that \cite{calvo2021state}, \cite{chow2017risk} have not extended their methods to modern model-free and model-based RL methods such as trust region policy optimization (TRPO)~\cite{schulman2015trust}, proximal policy optimization (PPO)~\cite{schulman2017proximal}, soft actor-critic (SAC)~\cite{haarnoja2018soft}, model-based policy optimization (MBPO)~\cite{janner2019trust}, probabilistic ensembles with trajectory sampling (PETS)~\cite{chua2018deep}.

We argue that our state augmentation approach allows viewing the safe RL problem from a different angle. Our state augmentation should be seen as a modification of an environment rather than an RL algorithm. It is implemented as a wrapper around OpenAI gym~\cite{brockman2016openai} environments, which allows ``saut\'eing'' \emph{any} RL algorithm. While we mostly test with model-free approaches (PPO, TRPO, SAC), the model-based methods are also ``saut\'eable'', which we illustrate on MBPO and PETS. We then demonstrate that a policy trained on one safety budget can generalize to other budgets. Further, if we randomly sample the initial safety state (i.e., the safety budget), then we can learn a policy \emph{for all safety budgets} in a set. 

\textbf{Related work.} Safe RL is based on the constrained MDP (c-MDP) formalism~\cite{altman1999constrained}, which has spawned many directions of research. A topic of considerable interest is safe exploration~\cite{turchetta2016safe, koller2018learning, dalal2018safe,  wachi2018safe, zimmer2018safe, bharadhwaj2020conservative}, where the goal is to ensure constraint satisfaction while exploring for policy improvement. Another line of research is to use classical control methods and concepts to learn a safe policy~\cite{chow2018lyapunov, chow2019lyapunov, berkenkamp2017safe, ohnishi2019barrier, cheng2019end, akametalu2014reachability, deanlqr2019, Fisac2019}. In these works, the authors also make strong prior assumptions such as partial knowledge of the model, and an initial safe policy to define a problem that can be solved. 

Besides classical control other tools were used in safe RL, e.g., a two-player framework with a task agent and a safety agent cooperating to solve the task~\cite{mguni2021desta}, a curriculum learning approach, where the teacher resets the student violating safety constraints~\cite{turchetta2020safe}, learning to reset if the safety constraint is violated~\cite{eysenbach2018leave}. 

While these are interesting topics of research, the classical RL setting with minimal assumptions is arguably more common in RL literature. A considerable effort was made in solving safe RL in the model-based setting~\cite{polymenakos2020safepilco, cowen2020samba, kamthe2018data}. In these works, the model was learned using Gaussian processes~\cite{GPbook, deisenroth2011pilco} allowing for an effective uncertainty estimation albeit with scalability limitations. Constrained versions of model-free RL algorithms such as TRPO, PPO, and SAC were also developed. For example, CPO~\cite{achiam2017constrained} generalized TRPO to a constrained case by explicitly adding constraints to the trust region update. \cite{raybenchmarking} were the first, to our best knowledge, to implement the Lagrangian version of PPO, SAC, and TRPO. These methods largely followed~\cite{chow2017risk}, who, however, considered an RL problem with a conditional-value-at risk (CVaR) constraints instead of average constraints. While Chow et al used classical policy gradient algorithms, recently this work was extended to PPO~\cite{cowen2020samba} and SAC~\cite{yang2021wcsac} algorithms. The Lagrangian method and CPO were improved in \cite{stooke2020responsive,yang2019projection}, respectively.  Finally, \cite{ding2020natural} proposed a natural policy gradient for c-MDPs.

\section{Vanilla RL and Safe RL with Constraints}
We first review basic RL and MDP concepts and adapt some definitions to our setting.
\begin{defn}
We define a Markov Decision Process (MDP) as a tuple $\cM = \langle  \cS, \cA, \cP, c, \gamma_c \rangle$, where $\cS$ is the state space; $\cA$ is the action space; $\gamma_c \in (0, 1)$ is the task discount factor, $\cP: \cS \times \cA \times \cS \rightarrow [0, 1]$, i.e., $\bms_{t+1} \sim p(\cdot | \bms_t , \bma_t)$, and $c:\cS \times \cA \rightarrow [0, +\infty)$ is the task cost. We associate the following optimization problem with the MDP:
\begin{equation}\label{prob:vanilla_mdp}
\begin{aligned}
        &\min\limits_{\pi}~ \E^\pi_\bms J_{\rm task},\,\, \\&J_{\rm task}\triangleq \sum\limits_{t=0}^\infty  \gamma_c^t c(\bms_t, \bma_t) | \bma_t \sim \pi,
 \end{aligned}
\end{equation}
where $\E^\pi_\bms$ is the mean over the transitions and the policy $\pi$. 
\end{defn}
The objective $J_{\rm task}$ implicitly depends on the initial state $\bms_0$ distribution and the policy $\pi$, but we drop this dependence to simplify notation. Throughout we assume that spaces $\cS$ and $\cA$ are continuous, e.g., $\cS\subset\R^{n_s}$ and $\cA\subset\R^{n_a}$, where $n_s$ and $n_a$ are the dimensions of the state and action spaces. 
We consider an infinite-horizon problem for theoretical convenience, in practice, however, we use the finite horizon (or episodes) setting as common in the RL literature~\cite{sutton2018reinforcement}. We also minimize the objective adapting the notation by~\cite{chow2017risk}. To solve the problem, the value functions are typically introduced:
\begin{equation*}
        \begin{aligned}
            V(\pi, \bms_0) = \E^\pi_\bms J_{\rm task}, \qquad\\
            V^\ast(\bms) = \min\limits_{\pi} V(\pi, \bms),
        \end{aligned}
\end{equation*}
where with a slight abuse of notation $\E^\pi_\bms$ is not averaging over the initial state $\bms_0$. In general, finding the representation of the optimal policy is not easy and the optimal policy can depend on the past trajectory, i.e., past state-action pairs. Remarkably, under some mild assumptions (see, for example, Appendix~\ref{app:theory}) the optimal policy solving Equation~\ref{prob:vanilla_mdp} depends only on the current state, i.e., $\bma_t \sim \pi(\cdot | \bms_t)$. This is a consequence of the Bellman equation which holds for the optimal value function $V^\ast_{\rm task}$:
\begin{equation*}
    V^\ast_{\rm task}(\bms) = \min\limits_{\bma\in\cA} \left(c(\bms, \bma) +  \gamma_c\mathbb{E}_{\bms'\sim p(\cdot| \bms, \bma)} V^\ast_{\rm task}(\bms')\right).
\end{equation*}
Additionally, using the Bellman equation we can also reduce the cost-to-go estimation to a static optimization rather than dynamic\footnote{There are still some non-stationary components in learning the value functions since the new data is constantly acquired.}. Both of these properties are at the heart of all RL algorithms. Now we can move on to constrained MDPs.
\begin{defn}\label{defn:constrained_mdp}
The constrained MDP (c-MDP) is a tuple $\cM_c = \langle  \cS, \cA, \cP, c, \gamma_c, l, \gamma_l \rangle$ where additional terms are the safety cost $l:\cS \times \cA \rightarrow [0, +\infty)$ and the safety discount factor $\gamma_l \in(0, 1)$. The associated optimization problem is:
\begin{subequations}
\begin{align}
        \min\limits_{\pi}~&\mathbb{E}^\pi_\bms  J_{\rm task} 
        \label{prob:average_objective}\\
     \text{s. t.: }&\mathbb{E}^\pi_\bms  J_{\rm safety}\ge 0,\,\, \\
     &J_{\rm safety} \triangleq d - \sum\limits_{t=0}^\infty  \gamma_{l}^t l(\bms_t, \bma_t).  \label{prob:average_constraint}
\end{align}\label{prob:average_constrained_mdp}
\end{subequations}
We will refer to the nonnegative value $d$ as \emph{the safety budget}. 
\end{defn}
 The average over the accumulated cost can be replaced by another statistic. For example, \cite{chow2017risk} proposed  to use conditional value at risk $\cvar_\alpha(X) = \min\limits_{\nu \in \R} \left(\eta + \frac{1}{1-\alpha} \mathbb{E}~\relu( X - \nu) \right)$ for $\alpha\in(0,1)$:
\begin{equation*}
\begin{aligned}
    &\min\limits_{\pi, \nu}~\mathbb{E}^\pi_\bms  J_{\rm task},\\
    &\text{s. t.: }\nu + \frac{1}{1-\alpha}\mathbb{E}^\pi_\bms ~\relu\left(d- J_{\rm safety} - \nu\right) \le d.  
\end{aligned}
\end{equation*}
In both cases, the problem can be solved using the Lagrangian approach, cf.~\cite{chow2017risk}. For example, in the case of Equation~\ref{prob:average_constrained_mdp} we can rewrite this problem as 
\begin{equation*} 
\begin{aligned}
    \min\limits_\pi &~~\E^\pi_\bms J_{\rm task} - \lambda^\ast  \E^\pi_\bms  J_{\rm safety}, \text{ where}\\
    \lambda^\ast &= \begin{cases}
  0 & \text{ if } \E^\pi_\bms  J_{\rm safety} \ge 0,\\
  +\infty &  \text{ if } \E^\pi_\bms  J_{\rm safety} < 0.
  \end{cases}
\end{aligned}
\end{equation*}
Instead of optimizing over an indicator $\lambda^\ast$, one formulates an equivalent min-max problem~\cite{bertsekas1997nonlinear}:
\[
    \min_{\pi} \max\limits_{\lambda \ge 0} \widehat J \triangleq \E^\pi_\bms  J_{\rm task} - \lambda  \E^\pi_\bms  J_{\rm safety}.
\]
Indeed, for every fixed policy $\pi$, the optimal $\lambda$ is equal to $\lambda^\ast$: if $\E^\pi_\bms J_{\rm safety}$ is nonnegative then there is no other choice but setting $\lambda$ to zero, if the $\E^\pi_\bms J_{\rm safety}$ is negative then maximum over $\lambda$ is $+\infty$~cf.~\cite{bertsekas1997nonlinear}. Thus we obtained the Lagrangian formulation of c-MDP. Now for actor-critic algorithms we need to estimate an additional value function $V_{\rm safety}(\pi, \bms_0) = \E^\pi_\bms J_{\rm safety}$ tracking the safety cost. 
Note that it is not clear if the optimal policy can be found using the commonly used representation $\bma_t \sim \pi(\cdot | \bms_t)$ since it is not clear what is the equivalent of the Bellman equation in the constrained case. Hence, the common policy representation (i.e., $\bma_t \sim \pi(\cdot | \bms_t)$) can create some limitations. Even intuitively the actions should depend on the safety constraint in some way, but they do not. 

\section{Saut\'e RL: Safety AUgmenTEd Reinforcement Learning}
\subsection{Main Idea in the Deterministic Case} \label{sec:main_idea}
We start by presenting the main idea in the deterministic case in order to simplify presentation. We consider a c-MDP with deterministic transitions, costs and one constraint:
\begin{align}
        \min\limits_{\pi}~&J_{\rm task}, \label{prob:cmdp_det}\\
        \text{s. t.: }& J_{\rm safety}\ge  0.\label{constraint:cmdp_det}
\end{align}
and its Lagrangian formulation 
\begin{equation}\label{prob:cmdp_lagrangian_det}
    \min_{\pi} \max\limits_{\lambda \ge 0} \widehat J \triangleq  J_{\rm task} - \lambda J_{\rm safety}.
\end{equation}

We adapt the ideas by~\cite{daryin2005nonlinear} to the RL case, and propose to reduce Problem~\ref{prob:cmdp_det} to an MDP. In particular, we remove the constraint by using state augmentation and then by reshaping the cost.

Let us take a step back and note that enforcing the constraint in Equation~\ref{prob:cmdp_det} is equivalent to enforcing the infinite number of the following constraints: 
\begin{equation}\tag{\ref{prob:cmdp_det}b$^\ast$}
\sum\limits_{k=0}^t \gamma_{l}^k l(\bms_k, \bma_k) \le d, \,\,\forall t\ge 0.
\end{equation}
This is because we assumed that the instantaneous cost is nonnegative and the accumulated safety cost cannot decrease. Therefore, if the constraint is violated at some time $t_v$, it will be violated for all $t \ge t_v$. It seems counter-intuitive to transform a problem with a single constraint into a problem with an infinite number of constraints. However, our goal here is to incorporate the constraints into the instantaneous task cost, thus taking into account safety while solving the task. This will be easier to perform while considering the constraint for all times $t$. To do so we track the remaining safety budget to assess constraint satisfaction at every time step. The remaining safety budget at time $t$ can be computed as $\bmw_t = d - \sum_{m=0}^t \gamma_l^m l(\bms_m, \bma_m)$, however, we will track a scaled version of $\bmw_t$, namely, $\bmz_{t} = \bmw_{t-1}/\gamma_l^{t}$, which has a time-independent update:
\begin{align*}
\bmz_{t+1} &=  
    (\bmw_{t-1} -  \gamma_l^t l(\bms_t, \bma_t))/ \gamma_l^{t+1}  \\
    &=(\bmz_t -  l(\bms_t, \bma_t) )/ \gamma_l, \\
    \bmz_0 &= \bmd.
\end{align*}
Since the variable $\bmz_t$ is Markovian (i.e., $\bmz_{t+1}$ depends only on $\bmz_t$, $\bma_t$ and $\bms_t$), we can augment our state-space with the variable $\bmz_t$. Now since we enforce the constraint $\bmz_t \ge 0$ for all times $t\ge 0$, we can reshape the instantaneous task cost to account for the safety constraint:
\begin{equation*}
    \begin{aligned}
         \widetilde c(\bms_t,\bmz_t, \bma_t) = 
                    \begin{cases}
                    c(\bms_t, \bma_t)  & \bmz_t \ge 0,  \\
                    +\infty & \bmz_t <0. 
                    \end{cases}
    \end{aligned}
\end{equation*}
Now we can formulate the problem \emph{without} constraints
\begin{equation}\label{prob:augmented_mdp}
        \min\limits_{\pi}~\sum\limits_{t=0}^\infty \gamma_c^t \widetilde c(\bms_t, \bmz_t, \bma_t).
\end{equation}

Note that the safety cost $l$, and the safety discount factor $\gamma_l$ are now part of the transition. The variable $\bmz_t$ can be intuitively understood as a risk indicator for constraint violation. The policy can learn to tread carefully for some values of $\bmz_t$ thus learning to interpret $\bmz_t$ as the distance to constraint violation. Note that the augmented state by~\cite{chow2017risk} tracks a value related to the CVaR computation rather than the remaining safety budget (see Appendix~\ref{app:state_augmentation_chow}), while the augmented state by~\cite{calvo2021state} is the Lagrange multiplier (see Appendix~\ref{app:state_augmentation_calvo}). In both cases, the augmented state by itself is not a very intuitive risk indicator for safety during an episode. Furthermore, in our case the safety budget $d$ is the initial safety state, which enables generalization across safety budgets as we show in our experiments. This was not done by~\cite{calvo2021state} and~\cite{chow2017risk}.\looseness=-1

To summarize, we have effectively showed the following:
\begin{thm} \label{thm:determinstic_equivalence}
An optimal policy for any of Equations~\ref{prob:cmdp_det},~\ref{prob:cmdp_lagrangian_det} and~\ref{prob:augmented_mdp} is also an optimal policy for all of these problems.
\end{thm}
 Next, we discuss how to deal with the infinity in the cost function and we generalize this idea to the stochastic case.

\subsection{Safety Augmented Markov Decision Processes}
The derivation in the general case are fairly similar at first and we obtain the following transition functions as above:
\begin{equation}
\begin{aligned}
     \bms_{t+1} &\sim p(\cdot | \bms_t , \bma_t), &&\bms_0 \sim {\cal S}_0, \\
     \bmz_{t+1} &= (\bmz_t -  l(\bms_t, \bma_t))/ \gamma_l, &&\bmz_0 = d.
\end{aligned}\label{eq:augmented_mdp}
\end{equation}

Note that the transitions in Equation~\ref{eq:augmented_mdp} are still Markovian and non-stationary, which simplifies the further algorithm development. 

In order to avoid dealing with infinite values in the cost $\widetilde c(\bms_t, \bmz_t, \bma_t)$, we introduce a proxy problem with a computationally friendlier cost:
\begin{equation}
\widetilde c_n(\bms_t,\bmz_t, \bma_t) = 
                    \begin{cases}
                    c(\bms_t, \bma_t)  & \bmz_t \ge 0,  \\
                    n & \bmz_t <0,
                    \end{cases}
\label{eq:bounded_augmented_reward}
\end{equation}
where $n$ is a hyper-parameter and introduce the Safety AUgmenTEd (Saut\'e) MDP $\widetilde \cM_n$.
\begin{defn}
Given a c-MDP $\cM_c = \langle  \cS, \cA,\cP, c, \gamma_c, l, \gamma_l \rangle$, we define a Safety Augmented Markov Decision Process (Saut\'e MDP) as a tuple $\widetilde \cM_n = \langle  \widetilde \cS, \cA, \widetilde\cP, \widetilde c_n, \gamma_c \rangle$, where $\widetilde \cS = \cS\times \cZ$; 
$\widetilde\cP: \widetilde\cS \times \cA \times \widetilde\cS  \rightarrow [0, 1]$ and defined in Equation~\ref{eq:augmented_mdp}, and
$\widetilde c_n:\widetilde\cS \times \cA \rightarrow [0,+\infty)$. We associate the following problem with Saut\'e MDP:
\begin{equation}\label{prob:saute_mdp}
        \min\limits_{\pi}~ \mathbb{E} \sum\limits_{t=0}^\infty \gamma_c^t \widetilde c_n(\bms_t, \bmz_t, \bma_t).
\end{equation}
\end{defn}
Now we need to answer two questions: a) what is the optimal representation of $\pi_n^\ast$; b) what is the relation of the MDPs $\widetilde \cM_n$ and $\widetilde \cM_\infty$. While the first question appears to be quite straightforward, the second requires revising results from stochastic optimal control. For readers' convenience, we summarize these results here. We make the following assumptions:

\textbf{A1.} The functions $\widetilde c_n(\bms, \bmz, \bma)$ are bounded, measurable, nonnegative, and lower semi-continuous on $\widetilde\cS \times \cA$;\\
\textbf{A2.} $\cA$ is compact;\\
\textbf{A3.} The transition $\cP$ is weakly continuous on $\widetilde \cS\times\cA$, i.e., for any continuous and bounded function $u$ on $\widetilde\cS$ the map $(\bms,\bmz, \bma) \rightarrow \int_{\widetilde \cS}u(\bmx,\bmy) \cP(d\bmx, d\bmy| \bms,\bmz, \bma)$ is continuous.

Note that these assumptions are rather mild and satisfied in many RL tasks. For instance, the task costs are often continuous and bounded, which is enough to satisfy Assumption 1 for the reshaped costs $\widetilde c_n$. Similarly, compactness of the action space is typically assumed in the RL setting as well. Finally, Assumption A3 is satisfied, if, for example, the transition function $\cP$ is a Gaussian with continuous mean and variance (cf.~\cite{arapostathis1993discrete}). 

Under these assumptions we can prove the following results (the proof is in Appendix~\ref{app:theory}).
\begin{thm} \label{thm:optimal_policy} Consider 
a Saut\'e MDP $\widetilde \cM_n$ satisfying Assumptions A1-A3 with the associated Equation~\ref{prob:saute_mdp}, then: 

a) for any finite $n$ the Bellman equation is satisfied, i.e., there exists a function $V_n^\ast(\bms, \bmz)$ such that 
\begin{equation*}
    V_n^\ast(\bms, \bmz) = \min\limits_{\bma \in \cA}\left(\widetilde c_n(\bms, \bmz, \bma)  + \gamma_c \mathbb{E}_{\bms',\bmz'} V_n^\ast(\bms', \bmz') \right),
\end{equation*}
where $\bms',\bmz'\sim \widetilde p(\cdot|\bms,\bmz, \bma)$. Furthermore, the optimal policy solving $\widetilde \cM_n$ has the representation $\bma \sim \pi_n^\ast(\cdot | \bms, \bmz)$;

b) the optimal value functions $V_n^\ast$ for $\widetilde\cM_n$ converge monotonically to $V_\infty^\ast$ --- the optimal value function for $\widetilde\cM_\infty$.
\end{thm}
The practical implication of our theoretical result is three-fold: a) we can use critic-based methods and guarantee their convergence under standard assumptions, b) the optimal policy is Markovian and depends on the safety budget, i.e., $\bma \sim \pi_n^\ast(\cdot | \bms, \bmz)$, and c) vanilla RL methods can be applied to solve $\widetilde \cM_n$. We finally stress that we can solve $\widetilde \cM_\infty$ only approximately by solving $\widetilde \cM_n$ for a large enough $n$. 

\subsection{Almost Surely Safe Markov Decision Processes}
While we derived an algorithm and studied the theoretical properties of our problem, we are yet to discuss what problem we are aiming to solve. 
Consider the following formulation for safe reinforcement learning.
\begin{defn} An \emph{almost surely constrained MDP} is a c-MDP $\cM_c$ with the associated optimization problem:
\begin{subequations}
\begin{align}
        \min\limits_{\pi(\cdot| \bms_t, \bmz_t)}~&\mathbb{E} J_{\rm task},\\
        \text{s.t.: }&\bmz_t \ge 0~~a.s.,~\forall t\ge 0, \label{con:as_safety}\\
        & \bmz_{t+1} = (\bmz_t - l(\bms_t, \bma_t)) / \gamma_l,  \\
        & \bmz_0 = d,\notag
\end{align}\label{prob:as_constrained_mdp}
\end{subequations}
where \text{a.s.} stands for ``almost surely'' (with probability 1). 
\end{defn}
Solving this problem using RL should deliver almost surely safe policies benefiting safety-critical applications. This formulation is much stronger than the average and the CVaR constrained ones. Tackling Problem~\ref{prob:as_constrained_mdp} directly seems to be impossible at the first sight. Remarkably, solving Saut\'e MDP $\widetilde \cM_\infty$ is equivalent to solving Equation~\ref{prob:as_constrained_mdp}. The equivalence should be understood in the following sense:

\begin{thm} \label{thm:AlmostSureSafeRL}
Consider a Saut\'e MDP $\widetilde \cM_\infty$ and Equation~\ref{prob:saute_mdp}. Suppose there exists an optimal policy $\pi^\ast(\cdot| \bms_t, \bmz_t)$ solving Equation~\ref{prob:saute_mdp} with a finite cost, then $\pi^\ast(\cdot| \bms_t, \bmz_t)$ is an optimal policy for Equation~\ref{prob:as_constrained_mdp}.
\end{thm}
\begin{proof}
The finite cost in $\widetilde \cM_\infty$ implies the satisfaction of Constraint~\ref{con:as_safety} almost surely. Now since the policy $\pi^\ast$ was obtained by minimizing the same objective as in Equation~\ref{prob:as_constrained_mdp} and satisfies Constraint~\ref{con:as_safety} almost surely, it also minimizes the objective in  Equation~\ref{prob:as_constrained_mdp}. 
\end{proof}

Above we showed that a policy solving Saut\'e MDP $\widetilde \cM_\infty$ and Equation~\ref{prob:saute_mdp} is actually policy solving Safe RL almost surely and Equation~\ref{prob:as_constrained_mdp}. We further showed that the optimal value function $V_n^\ast$ for $\widetilde \cM_n$ converges to the optimal value function $V_\infty^\ast$ for $\widetilde \cM_\infty$ under some mild conditions. While in many practical scenarios, this means that the policy for $\widetilde \cM_n$ for large $n$ approximates well the policy solving Safe RL almost surely, there may be some specific settings when this is not the case. Further assumptions can improve our analysis, however, we refer the reader to~\cite{hernandez1992discrete} for a detailed discussion.

\begin{figure}[ht]
     \centering
     \begin{subfigure}[b]{0.43\columnwidth}
     \centering
     \includegraphics[width=0.99\textwidth]{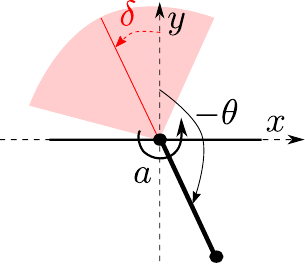} 
     \caption{Safe pendulum swing-up.}
     \label{mt_fig:safe_pendulum}
     \end{subfigure}~
     \begin{subfigure}[b]{0.43\columnwidth}     
     \centering
     \includegraphics[width=0.99\textwidth]{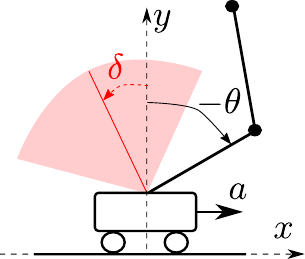} 
     \caption{Safe double pendulum}
     \label{mt_fig:safe_double_pendulum}
     \end{subfigure}
     \vspace{0.1mm}
     \begin{subfigure}[b]{0.43\columnwidth}     
     \centering
     \includegraphics[width=0.99\textwidth]{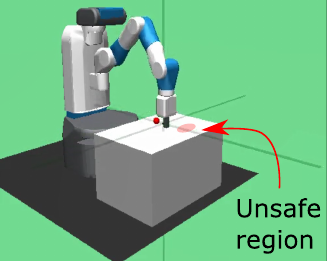} 
     \caption{Safe reacher}
     \label{mt_fig:safe_fetch_reacher}
     \end{subfigure}~
     \begin{subfigure}[b]{0.43\columnwidth}     
     \centering
     \includegraphics[width=0.99\textwidth]{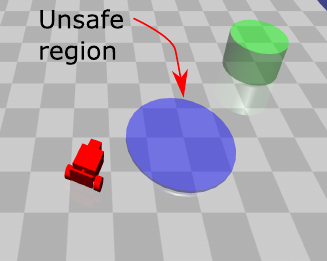} 
     \caption{Safety gym}
     \label{mt_fig:static_env_gym}
     \end{subfigure}
    \caption{Panels a and b: safe pendulum environments. In both cases, $\theta$ - is the angle from the upright position, $a$ is the action, $\delta$ - is the unsafe pendulum angle, the safety cost is the distance toward the unsafe pendulum angle, which is incurred only in the red area. Panel e: safe reacher:  robot needs to avoid the unsafe region (in red). Panel d: a safety gym environment: robot needs to reach the goal (in green) while avoiding the unsafe region (in blue).}
    \label{mt_fig:envs}
\end{figure}

\section{Experiments}
\textbf{Environments.} We demonstrate the advantages and the limitations of our approach on three OpenAI gym environments with safety constraints (pendulum swing-up, double pendulum balancing, reacher) and the OpenAI safety gym environment (schematically depicted in Figure~\ref{mt_fig:envs}). In the environment design we follow previous work by~\cite{kamthe2018data},~\cite{cowen2020samba},~\cite{yang2021wcsac} and delegate the details to Appendix~\ref{app:environments}.

\begin{figure*}[ht!]
    \centering
     \begin{subfigure}[b]{0.28\textwidth}
     \centering
     \includegraphics[width=0.99\textwidth]{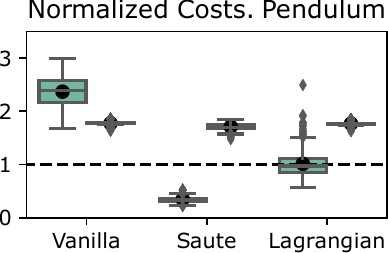} 
     \caption{PPO}
     \label{mt_fig:ppo_single}
     \end{subfigure}
     \begin{subfigure}[b]{0.28\textwidth}
     \centering
     \includegraphics[width=0.99\textwidth]{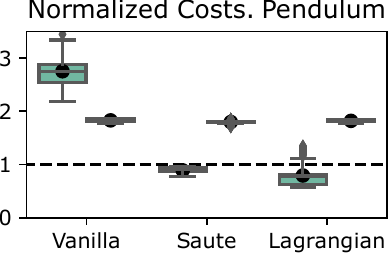}
     \caption{SAC}
     \label{mt_fig:sac_single}
     \end{subfigure}
     \begin{subfigure}[b]{0.39\textwidth}
     \centering
     \includegraphics[width=0.99\textwidth]{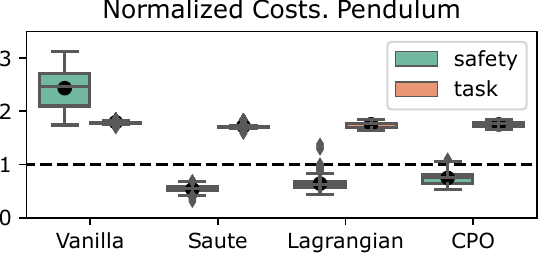} 
     \caption{TRPO}
     \label{mt_fig:trpo_single}
     \end{subfigure}
\caption{Saut\'e RL as a plug-n-play approach. 
Box plots for normalized safety(on the left) and task (on the right) costs for SAC, PPO and TRPO-type algorithms on pendulum swing-up environment with the safety budget $30$ after $300$ epochs of training. In all figures  the task cost are divided by $-100$ and the safety costs are divided by $30$, the dashed lined indicate the safety threshold. In all cases, ``saut\'eed'' algorithms deliver safe policies with probability 1 (outliers whiskers do not cross the dashed line), while Lagrangian methods and CPO have trajectories violating the safety constraints. For task costs the higher values are better.}
    \label{mt_fig:single_pendulum}
\end{figure*}
\begin{figure*}[ht]
     \centering
     \begin{subfigure}[b]{0.39\textwidth}
     \centering
     \includegraphics[width=0.99\textwidth]{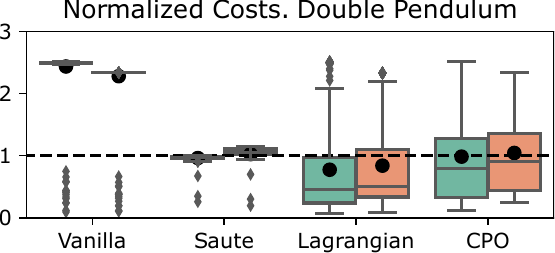}
     \caption{TRPO}
     \label{mt_fig:double_pendulum}
     \end{subfigure}          
     \begin{subfigure}[b]{0.22\textwidth}
     \centering
     \includegraphics[width=0.99\textwidth]{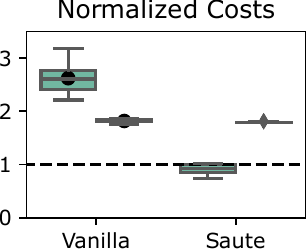}
     \caption{MBPO}
     \label{mt_fig:mbpo}
     \end{subfigure}     
     \begin{subfigure}[b]{0.22\textwidth}
     \centering
     \includegraphics[width=0.99\textwidth]{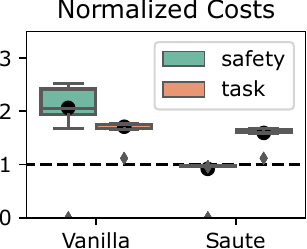}
     \caption{PETS}
     \label{mt_fig:pets}
     \end{subfigure}
\caption{Saut\'e TRPO on the double pendulum environment (panel a) and Saut\'e MBRL on the pendulum swing-up environment (Panels b and c). Box plots for safety costs (on the left) and normalized task (on the right). Panel a: The task costs are divided by $-80$, while the safety costs by $40$ with the safety budget $40$. Panels b and c: The task costs are divided by $100$, while the safety costs by $30$ with $30$ being the safety budget. In all plots dashed lines indicate the safety threshold. For task costs the higher values are better.}
    \label{mt_fig:res}
\end{figure*}

\textbf{Implementation.} The main benefit of our approach to safe RL is the ability to ``saut\'e'' \emph{any RL algorithm}. This is because we do not need to change the algorithm itself (besides some cosmetic changes), but create a wrapper around the environment. The implementation is quite straightforward and the only ``trick'' we used is normalizing the safety state by dividing it by the safety budget:
\begin{equation*}
\begin{aligned}
    \bmz_{t+1} &= (\bmz_t - l(\bms_t, \bma_t) / d) / \gamma_l, \\
    \bmz_0 &= 1.
\end{aligned}
\end{equation*}
Hence the variable $\bmz_t$ is always between zero and one. Note this does not affect our theoretical results. The reset and step functions have to be overloaded to augment the safety state and shape the cost as in Equation~\ref{eq:bounded_augmented_reward}. More details on ``saut\'eed'' environment implementation are available in Appendix~\ref{app:detailed_implementation_details}. We used safety starter agents~\cite{raybenchmarking} as the core implementation for model-free methods, their Lagrangian versions (PPO, TRPO, SAC), and CPO. We use the hyper-parameters listed in Appendix~\ref{app:default_hyperparameters}. For our model-based implementations, we used~\cite{Pineda2021MBRL} as the core library, which has PyTorch implementation of MBPO~\cite{janner2019trust}, and PETS~\cite{chua2018deep}. Finally, we implemented a CVaR constrained safe RL based on the safety starter agents. We discuss our implementation in Appendix~\ref{app:cvar}. Our implementations are available online~\cite{sootla_saute_2022_git}. 

\textbf{Evaluation protocols.} In all our experiments we used $5$ different seeds, we save the intermediate policies and evaluate them on $100$ different trajectories in all our figures and tables. One exception is the evaluation of PETS, for which we used $25$ trajectories. Note that in all the plots we use returns based on the original task costs $c$, not the reshaped task costs $\widetilde c_n$ to evaluate the performance. In all our experiments we set the safety discount factor for Saut\'e RL equal to one, while the safety discount factor for other algorithms varies. We also use box-and-whisker plots with boxes showing median, $q_3$ and $q_1$ quartiles of the distributions ($75$th and $25$th percentiles, respectively), whiskers depicting the error bounds computed as $1.5 (q_3 - q_1)$, as well as outliers, e.g., points lying outside the whisker intervals~ \cite{Waskom2021}. We add black dots to the plots which signify the mean. We use box-and-whisker plots so that we can showcase the outliers and the percentiles, which are important criteria for the evaluation of almost surely constraints.

\textbf{Saut\'e RL is an effective plug-n-play approach.} We first demonstrate that Saut\'e RL can be easily applied to both on-policy  (PPO, TRPO) and off-policy algorithms (SAC) without significant issues. We run all these algorithms on the pendulum swing-up environment. We test the policies with the initial state sampled around the downright position of the pendulum. The results in Figure~\ref{mt_fig:single_pendulum} indicate that PPO, TRPO, and SAC can be effectively ``saut\'eed'' and deliver policies safe almost surely (i.e., with probability one and all trajectories satisfy the constraint). Note that the difference in behavior for Trust Region-based algorithms is the smallest, while Saut\'e SAC delivers the best overall performance. We also present the evaluation during training in Figures~\ref{fig:saute_RL_single},~\ref{fig:trpo_single_pendulum},~\ref{fig:ppo_single_pendulum} and~\ref{fig:sac_single_pendulum} in Appendix.

\textbf{Saut\'e Model-Based RL.} We proceed by ``saut\'eing'' MBRL methods: MBPO and PETS. As the results in Figures~\ref{mt_fig:mbpo}, and~\ref{mt_fig:pets} suggest, we lose some performance, but guarantee safety in both cases. Remarkably we could ``saut\'e'' both MPC and policy-based methods without significant issues.

As we have demonstrated the plug-n-play nature of our approach for model-free and model-based methods, in all further experiments we evaluate our method on Trust Region-based algorithms only, i.e., Vanilla TRPO, its variants, and CPO. We did so because Saut\'e TRPO has a lower gap in performance with Lagrangian TRPO than Saut\'e SAC and Saut\'e PPO have with their Lagrangian versions. However, a fair evaluation against CPO was also appealing.

\textbf{``Safety on average'' can be very unsafe even in deterministic environments.} While safety on average is less restrictive than safety almost surely, in some situations safety on average can lead to unwanted behaviors. We design the safe double pendulum environment in such a way that task and safety costs are correlated. Hence restricting the safety cost leads to restricting the task cost and forces the Lagrangian algorithm to balance a trade-off between these objectives. Further, the constraints on ``average'' allow Lagrangian TRPO and CPO to learn the policies that prioritize minimizing task costs for some initial states and minimizing safety costs for other initial states. While the constraint is satisfied on average, the distributions of task and safety costs for both CPO and Lagrangian TRPO have a large variance (see Figure~\ref{mt_fig:double_pendulum}). Further, some outliers have similar behavior to the Vanilla TRPO. Saut\'e TRPO on the other hand achieves the best overall performance. We plot the evaluation curves during training in Appendix in Figure~\ref{fig:trpo_double_pendulum}.

\begin{figure}[ht]
     \centering
     \begin{subfigure}[b]{0.85\columnwidth}
     \centering
     \includegraphics[width=0.99\textwidth]{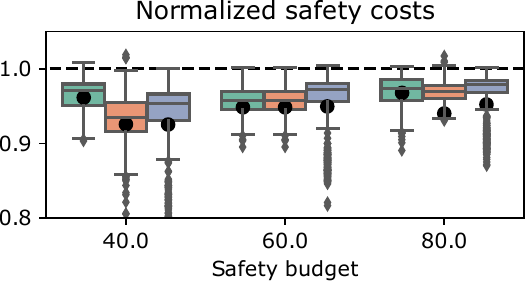}
     \caption{Safety costs}
     \label{mt_fig:generalizing_safety_cost}
     \end{subfigure}
     \begin{subfigure}[b]{0.85\columnwidth}
     \centering
     \includegraphics[width=0.99\textwidth]{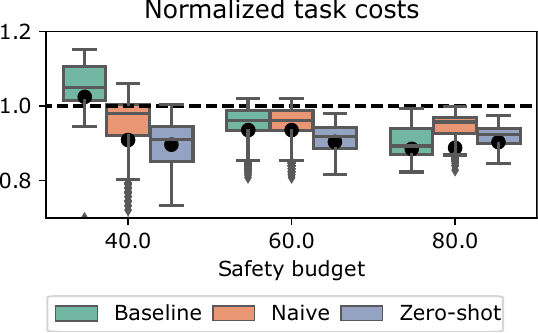}
     \caption{Task costs}
     \label{mt_fig:ggeneralizing_task_cost}
     \end{subfigure}          
\caption{The normalized task and safety costs for generalization across safety budgets. The task costs are divided by $-2 d$, while the safety costs by $d$ with $d$ being the safety budget, hence dashed lines indicate the safety threshold. The \emph{baseline} policies are trained and evaluated on the same safety budges; the \emph{na\"ive} approach trains policies on the safety budget $60$ and the \emph{zero-shot} approach trains policy by sampling the safety budget from the interval $[5,100]$. For task costs the higher values are better.}        
    \label{mt_fig:generalizing}
\end{figure}
\textbf{Generalization across safety budgets.} Since the safety budget $d$ enters the problem formulation as the initial value of the safety state, we can generalize to a different safety budget \emph{after training} by changing the initial safety state. We train three separate set of policies for safety budgets $40$, $60$, $80$, we then take the policies trained for the safety budget $60$ and evaluate them on the safety budgets $40$ and $80$. The test results of this \emph{na\"ive} generalization approach over 5 different seeds are depicted in Figure~\ref{mt_fig:generalizing} showing that the na\"ive generalization approach has a similar performance with policies explicitly trained for budgets $40$ and $80$. We further train another set of policies with the initial safety state uniformly sampled from the interval $[5, 100]$. When it comes to safety constraint satisfaction this \emph{zero-shot} approach outperforms the na\"ive approach, which has some outlier trajectories for safety budgets $40$ and $80$ (see Figure~\ref{mt_fig:generalizing}). 

\begin{figure}[t]
     \centering
 \begin{subfigure}[t]{0.8\columnwidth}
     \centering
      \includegraphics[width=0.99\columnwidth]{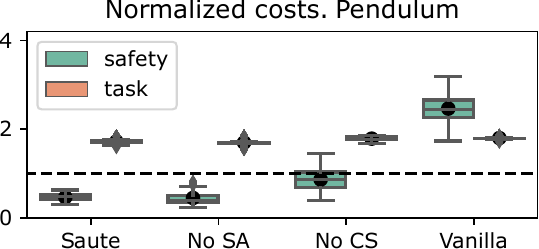}     \caption{Ablation on components. Pendulum Swing-up}
     \label{mt_fig:components_single}
 \end{subfigure}
 
     \vskip 2mm
     
      \begin{subfigure}[t]{0.8\columnwidth}
     \centering
     \includegraphics[width=0.99\columnwidth]{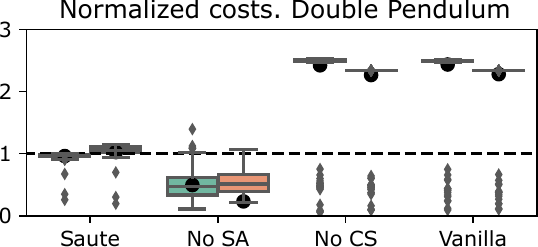}
     \caption{Ablation on components. Double Pendulum}
     \label{mt_fig:components_double}
     \end{subfigure}
     
     \vskip 2mm

     \begin{subfigure}[t]{0.95\columnwidth}
     \centering
     \includegraphics[width=0.99\columnwidth]{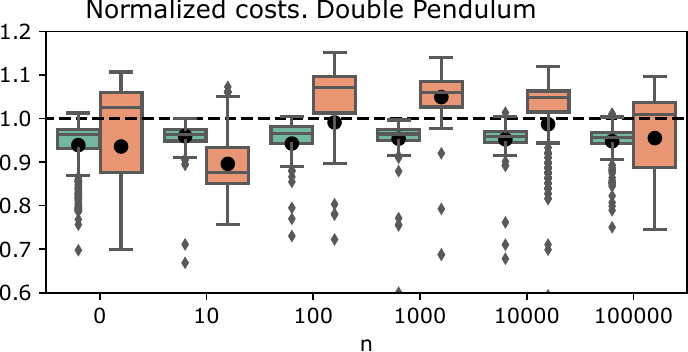}
     \caption{Task cost shaping}
     \label{mt_fig:reward_shaping}
     \end{subfigure}
\caption{Normalized safety (on the left) and  task (on the right) costs in ablation studies for Saut\'e TRPO. Panels a and b: ablation on components for double pendulum (Panel a) and pendulum swing-up (Panel b). ``No SA'' stands for no state augmentation, ``No CS'' stands for no cost shaping. Panel c: varying values $n$ in reshaped costs $\widetilde c_n$ for the safety budget $d = 40$. The numerical values can be found in Table~\ref{tab:reward_shaping} in Appendix. In all plots the task costs are divided by $-80$ (i.e., higher values are better), while the safety costs by $40$, dashed lines indicate the safety threshold.}          \label{mt_fig:ablation_cost}
\end{figure}
\textbf{Ablation.} Since we have only two main components (cost shaping and state augmentation) our ablation study is rather straightforward. We perform evaluations on the double pendulum environment with the safety budget set to $40$. According to the results in Figure~\ref{mt_fig:components_double} removing cost shaping produce results similar to Vanilla TRPO, which is expected. Removing state augmentation leads to a significant deterioration in task cost minimization and the safety costs are much lower. It appears that the trained policy hedges its bets by not using the whole safety budget, which leads to the task cost increase. 
Interestingly, in the pendulum swing-up case, removing the state augmentation is not catastrophic, see Figure~\ref{mt_fig:components_single}. This is because we evaluate the policy while sampling the initial states near the same initial position. Hence the safety states are similar at every time for different trajectories. This allows the algorithm without state augmentation to still produce competitive results. This experiment shows that the effect of state augmentation may be easily overlooked.

All the parameters in a ``saut\'eed'' algorithm are the same as in its ``vanilla'' version except for the parameter $n$ in the cost $\widetilde c_n$. Hence we only need to perform the second ablation with respect to $n$. In all our previous experiments with the double pendulum, we set $n=200$, and here we test different values $n$ for the safety budget $40$ and present the evaluation results after $600$ epochs of training in Figure~\ref{mt_fig:reward_shaping}. Increasing $n$ from $0$ to $100$ improves the performance of Saut\'e RL by decreasing the number of outliers in safety cost distributions. However, increasing the value of $n$ to $10000$ and $100000$ leads to additional outliers in safety and task cost distributions. We attribute this to numerical issues as the task costs $c$ take values between zero and one, and the large value of $n$ can lead to numerical issues in training. We note, however, that the differences between the numerical values are not large, which can be verified in Table~\ref{tab:reward_shaping} in Appendix.

\begin{figure}[t]
     \centering
      \begin{subfigure}[b]{0.99\columnwidth}
     \centering
     \includegraphics[width=0.99\columnwidth]{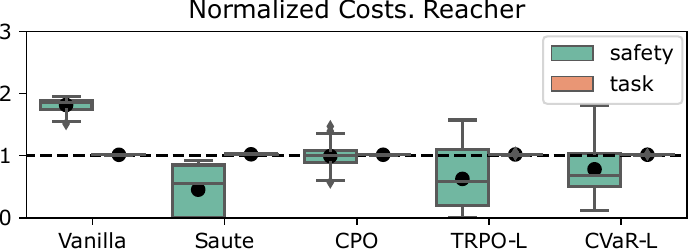}
     \caption{Reacher}
     \label{mt_fig:reacher}
     \end{subfigure}
     
     \vskip 2mm
     
      \begin{subfigure}[b]{0.99\columnwidth}
     \centering
     \includegraphics[width=0.99\columnwidth]{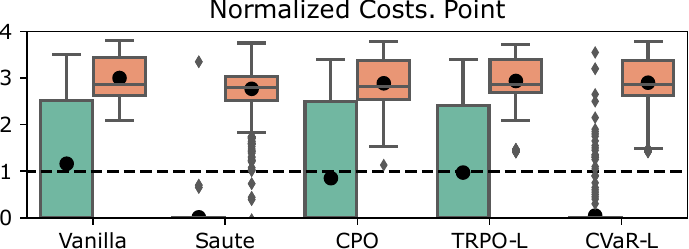}
     \caption{Safety gym. Point}
     \label{mt_fig:safety_gym_point}
     \end{subfigure}
          
     \vskip 2mm

  \begin{subfigure}[b]{0.99\columnwidth}
     \centering
     \includegraphics[width=0.99\columnwidth]{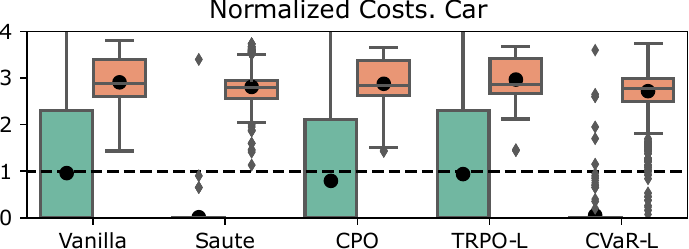}
  \caption{Safety gym. Car}
     \label{mt_fig:safety_gym_car}
     \end{subfigure}
\caption{Normalized safety (on the left) and task (on the right) costs in Reacher and Safety Gym. Panel a: Results for the reacher environment with the safety budget $10$, where the task costs are divided by $60$, while the safety costs by $10$. Panel b: Results for the safety gym point environment with the safety budget $20$, where the task costs are divided by $-1$ and the safety costs are divided by $20$. Dashed lines indicate the safety threshold. For task costs the higher values are better.} \label{mt_fig:boxplots}
\end{figure}
\textbf{Further Experiments}. We then test Saut\'e TRPO on Reacher Environment as well as a Safety Gym environment with Car and Point robots. Results in Figure~\ref{mt_fig:reacher} are quite similar to the results on the pendulum swing-up and double pendulum environments, where both TRPO Lagrangian (TRPO-L) and CPO delivered the policies safe on average, but not safe almost surely as Saut\'e TRPO. We also compared our implementation of the CVaR constrained problem with $\alpha=0.01$, which is denoted as CVaR-L and is an approximation of almost surely safety constraints (see Appendix~\ref{app:cvar}). 
While the reacher task is closer in nature to pendulums, the safety gym is quite a complicated environment. Every environment in the safety gym has dozens of states ($46$ for the point environment and $56$ for the car environment) due to the inclusion of LIDAR measurements. Finally, the instantaneous task costs are shaped so that their values are close to zero, which is tailored for TRPO and PPO-like algorithms. This makes our approach a bit harder to use since we reshape the costs. Nevertheless, the results in Figure~\ref{mt_fig:safety_gym_point} indicate that Saut\'e RL delivers a safe set of policies with only a few outliers violating the constraints. Most of the trajectories have the same returns as TRPO Lagrangian, but the average task cost is brought down by a few outliers. It appears that in these outlier trajectories the ``saut\'eed'' policies prioritize safety over task costs. While TRPO-L and CPO produce on average better task costs, the safety constraints are being violated on a rather regular basis. Note that in our experiments the safety budget is chosen to be the average incurred cost by Vanilla TRPO and neither TRPO-L nor CPO decreases the variance of the safety cost. While TRPO-L and CPO lower the average cost, the number of outlier trajectories remains quite high. Finally, CVaR-L fails for the Reacher environment but appears to be a somewhat competitive approximation for Saut\'e RL for the Point and Car environments. Our results suggest, however, that approximating almost surely constraints with CVaR constraints can still result in a significant number of outliers. 

We further compared our method to a more advanced baseline where a PID controller is used to update the Lagrangian multiplier~\cite{stooke2020responsive}. Our results suggest that PID Lagrangian offers some level of improvement in training (especially stability of the training process). However, the end performance is similar to TRPO-L in our environments  (see results in Appendix~\ref{app:pidl-comparison}). This is because we consider a completely different problem formulation to PID Lagrangian and Lagrangian approaches. 

\section{Conclusion}
We presented an approach to safe RL using state augmentation, which we dubbed Saut\'e RL. The key difference of our approach is that we ensure the safety constraints almost surely (with probability one), which is desirable in many applications. Even, in deterministic environments having an ``average'' constraint can lead to unwanted effects, when the safety cost is high for some initial states and is low for other initial states at the same time. Our approach deals with this case by ensuring that the same constraint is satisfied for all initial states. We showed that state augmentation is essential in some environments for optimality, which supports our theoretical results implying that the optimal policy depends on the safety state. The constraint satisfaction almost surely is a very strong criterion and in some applications, it can be too restrictive. This is, however, a design choice and application dependent, now let us discuss specific pros and cons.

\textbf{Advantages.} Saut\'e RL has a plug-n-play nature, which allows for straightforward extensions as we demonstrated by saut\'eing PPO, TRPO, and SAC. Furthermore, we used saut\'eed environments in the model-based RL setting (MBPO and PETS) as well. We showed that Saut\'e RL generalizes across safety budgets and can learn safe policies for all safety budgets simultaneously. This feature is enabled by the architecture of our state augmentation. Since the remaining safety budget is the initial state, we can randomly sample different safety budgets at the beginning of the episode. At the test time, the safety budget could be set in a deterministic fashion to evaluate specific Safe RL problems.

\textbf{Limitations.} We have not treated the case with multiple constraints, which can bring some difficulties for cost re-shaping. 
Further, ``saut\'eing'' an MDP increases the state-space by the number of constraints. Therefore, the dimension of value functions and policy grows and potentially can lead to scalability issues. While this is a common issue in constrained problems, using a Lagrangian approach can be more sample efficient. Since the theoretical sample efficiency estimates usually depend on the number of states~\cite{mania2018simple}, it would be interesting to find means to counteract this loss of efficiency. This potentially can be done in the setting where the safety cost function is given, by exploiting the known safety state transitions.

Saut\'e RL does not currently address the problem of constraint violation during training, which is still a major problem in safe RL. However, the combination of Saut\'e RL and methods for addressing such a problem could be an interesting direction for future work as well. After all, the cost-to-go in Saut\'e RL does incorporate potential safety violations and this information can potentially be used for safe training. 

\textbf{Future Work.} Besides addressing the limitations, it would be interesting to extend our approach to model-based algorithms applicable to high-dimensional environments including PlaNet~\cite{hafner2019dream}, Dreamer~\cite{hafner2019dream}, Stochastic Latent Actor Critic~\cite{lee2020stochastic} etc. Furthermore, it would be interesting to evaluate the effect of safety state augmentation on average and CVaR constrained problems.

\bibliography{safe_rl}
\bibliographystyle{icml2022}

\setcounter{equation}{0}
\setcounter{thm}{0}
\setcounter{defn}{0}
\setcounter{figure}{0}
\setcounter{table}{0}
\setcounter{section}{0}

\renewcommand{\thethm}{A\arabic{thm}}
\renewcommand{\thedefn}{A\arabic{defn}}
\renewcommand{\theprop}{A\arabic{prop}}
\renewcommand{\thelem}{A\arabic{lem}}
\renewcommand{\thesection}{A\arabic{section}}

\renewcommand{\theequation}{A\arabic{equation}}
\renewcommand{\thetable}{A\arabic{table}}
\renewcommand{\thefigure}{A\arabic{figure}}
\renewcommand\thesubfigure{(\alph{subfigure})}

\appendix 

\section{Theoretical analysis}\label{app:theory}
Proof of Theorem~\ref{thm:optimal_policy} follows form the results by~\cite{hernandez1992discrete} and~\cite{hernandez2012discrete}, which we reproduce and condense for readers' convenience. We cover the conditions for the existence of the Bellman equation and optimal policies in Appendix~\ref{app:bellman}, and we discuss the convergence of a sequence of MDPs to a limit MDP in Appendix~\ref{app:sequence_bellman}
and discuss the application of these results to our case in Appendix~\ref{app:saute_discussion}.

\subsection{MDPs, Optimality and Bellman equation} \label{app:bellman}
 Consider an MDP $\cM = \{\cS, \cA, \cP, c, \gamma_c\}$ with an action set defined for every state $\bma \in \cA(\bms)$, where $\cA$ are non-empty sets. The set 
\begin{equation*}
    \K = \{(\bms, \bma) | \bms \in \cS, \bma \in \cA(\bms)\}
\end{equation*}
of admissible state-action pairs is assumed to be a Borel subset of $\cS \times \cA$. We will need the following definitions: 
\begin{defn}
A function $u$ is \emph{inf-compact} on $\K$ if the set $\{\bma \in \cA(\bms) | u(\bms, \bma) \le r\}$ is compact for every $\bms\in\cS$ and $r\in\R$.

A function $u$ is \emph{lower semi-continuous (l.s.c.)} in $\cS$ if for every $\bms_0 \in\cS$ we have $\liminf\limits_{\bms\rightarrow\bms_0} u(\bms) \ge \bms_0$.

A set-valued function $\bms \rightarrow \cA(\bms)$ is \emph{lower semi-continuous (l.s.c.)}, if for any $\bms_n \rightarrow \bms$ in $\cS$ and $\bma\in \cA(\bms)$, there are $\bma_n\in\cA(\bms_n)$ such that $\bma_n\rightarrow \bma$.

A distribution $\cQ(\bmy | \bms, \bma)$ is called \emph{weakly continuous}, if for any continuous and bounded function $u$ on $\cS$ the map $(\bms, \bma) \rightarrow \int_{\cS}u(\bmy) \cQ(d\bmy| \bms, \bma)$ is continuous on $\K$.
\end{defn}

Let the value functions be denoted as follows:
\begin{align*}
    V(\pi, \bms_0) &= \E_{\bms}^\pi \sum\limits_{t=0}^\infty \gamma_c ^t c(\bms_t, \bma_t),\\
    V^\ast( \bms) &\triangleq \inf\limits_\pi V(\pi, \bms),
\end{align*}
where $\E^\pi_\bms$ stands for the average with action sampled according to the policy $\pi$ and the transitions $\cP$. We also define the Bellman operator:
\begin{equation*}
        T v(\bms) = \min\limits_{\bma \in \cA(\bms)} \left[c(\bms, \bma) + \gamma \int v(\bmy) \cP(d\bmy| \bms, \bma)\right],
\end{equation*}
acting on value functions. We also make the following assumptions:
\begin{itemize}
    \item[B1.] The function $c(\bms, \bma)$ is bounded, measurable on $\K$, nonnegative, lower semi-continuous and inf-compact on $\K$;
    \item[B2.] The transition law $\cP$ is weakly continuous;
    \item[B3.] The set valued map $\bms \rightarrow \cA(\bms)$ is lower semi-continuous;

\end{itemize}
We summarize Theorems 4.2 and~4.6 by~\cite{hernandez1992discrete} in the following result:
\begin{prop}
Suppose an MDP $\cM =  \{\cS, \cA, \cP, c, \gamma_c\}$ satisfies Assumptions B1-B3. Then: 

a) The optimal cost function $V^\ast$ satisfies the Bellman equation, i.e., $T V^\ast = V^\ast$ (Theorem 4.2);

b) The policy $\pi^\ast$ is optimal (i.e., $V(\pi^\ast, \cdot) = V^\ast(\cdot)$) if and only if $V(\pi^\ast, \cdot) = T V(\pi^\ast, \cdot)$ (Theorems 4.2 and 4.6).
\end{prop}

\cite{hernandez1992discrete} proved these results under milder conditions on the cost function than the boundedness condition we use. However, \cite{hernandez1992discrete} also had an assumption on the existence of a feasible policy, i.e., they assumed that there exists a policy $\widehat \pi$ such that $V(\widehat \pi, \bms) < \infty$ for each $\bms \in \cS$. This, however, follows from the boundedness of the cost function. 

\subsection{Limit of a sequence of MDPs} \label{app:sequence_bellman}
Consider now a sequence of MDPs $\cM_n = \{\cS, \cA, \cP, c_n, \gamma_c\}$, where without loss of generality we will write $c \triangleq c_\infty$ and $\cM \triangleq \cM_\infty$. Consider now a sequence of value functions $\{V_n^\ast\}_{n=0}^\infty$:
\begin{align*}
    V_n(\pi, \bms_0) &= \E_{\bms}^\pi \sum\limits_{t=0}^\infty \gamma_c ^t c_n(\bms_t, \bma_t),\\ 
    V_n^\ast(\bms) &\triangleq \inf\limits_\pi V_n(\pi, \bms).
\end{align*}
The ``limit'' value functions (with $n = \infty$) we still denote as follows:
\begin{align*}
    V(\pi, \bms_0) &= \E_{\bms}^\pi \sum\limits_{t=0}^\infty \gamma_c ^t c(\bms_t, \bma_t),\\
    V^\ast( \bms) &\triangleq \inf\limits_\pi V(\pi, \bms).
\end{align*}

We also define the sequence of Bellman operators 
\begin{align*}
    T_n v(\bms) &= \min\limits_{\bma \in \cA(\bms)} \left[c_n(\bms, \bma) + \gamma \int v(\bmy) \cP(d\bmy| \bms, \bma)\right],\\
    T v(\bms) &= \min\limits_{\bma \in \cA(\bms)} \left[c(\bms, \bma) + \gamma \int v(\bmy) \cP(d\bmy| \bms, \bma)\right].
\end{align*}
In addition to the previous assumptions, we make an additional one, while modifying Assumption B1:
\begin{itemize}
    \item[B1'.] For each $n$ the functions $c_n(\bms, \bma)$ are bounded, measurable on $\K$, nonnegative, lower semi-continuous and inf-compact on $\K$;
    \item[B4.] The sequence $\{c_n(\bms, \bma)\}_{n=0}^\infty$ is such that $c^n \uparrow c$;
\end{itemize}

We reproduce Theorem 5.1 by~\cite{hernandez1992discrete} in the following proposition:
\begin{prop}
Suppose MDPs $\cM_n =  \{\cS, \cA, \cP, c_n, \gamma_c\}$ satisfy Assumptions B1', B2 - B4, then the sequence $\{V_n^\ast\}$ is monotonically increasing and converges to $V^\ast$.
\end{prop}

Note that we do not require the cost function $c$ to be bounded. This, however, comes at a price in that we cannot generally claim that $T V^\ast = V^\ast$. To ensure this property we need additional assumptions (see~\cite{hernandez1992discrete}). 

\subsection{Proof of  Theorem~\ref{thm:optimal_policy}}\label{app:saute_discussion}
Coming back to the Saut\'e MDP, recall that we define the following cost function for $\widetilde \cM_n$:
\begin{equation*}
    \begin{aligned}
         \widetilde c_n(\bms_t, \bmz_t, \bma_t) = 
                    \begin{cases}
                    c(\bms_t, \bma_t)  & \bmz_t \ge 0,  \\
                    n  & \bmz_t <0. 
                    \end{cases}
    \end{aligned}
\end{equation*}
Therefore, to prove Theorem~\ref{thm:optimal_policy}
we need to verify that Saut\'e MDP $\widetilde \cM_n$ satisfying Assumptions A1-A3 also satisfies Assumptions B1', B2-B4. According to Assumptions A1-A2, we consider bounded, continuous costs $c$ with compact action space $\cA$, hence Assumptions B1', B3, and B4 are satisfied. Assumptions B2 and A3 are identical. Note that, if the transition function $\cP$ is a Gaussian with continuous mean and variance, then Assumption A3 is satisfied  (cf.~\cite{arapostathis1993discrete}).

\section{State augmentation techniques}\label{app:state_augmentation}
\subsection{By Daryin and Kurzhanski}\label{app:state_augmentation_daryin}
\cite{daryin2005nonlinear} considered the classical control problem, i.e., the model is assumed to be known and the goal is to compute the optimal policy. They consider the following model:
\begin{equation*}
    \dot \bmx(t) = \bmA(t) \bmx(t) + \bmB(t) \bmu(t) + \bmC(t) \bmv(t),
\end{equation*}
where $\bmx(t)$ is the state, $\bmu(t)$ is the control signal and $\bmv(t)$ is an unknown disturbance. Both controls and disturbances are subject to hard bounds:
\begin{equation*}
    \bmu(t) \in \cU(t), \qquad \bmv(t) \in \cV(t),
\end{equation*}
where the time-varying sets $\cU(t)$ and $\cV(t)$ are also known. The controls are also subject to soft constraints:
\begin{equation*}
    \int\limits_{t_0}^{t_1} \|\bmu(t)\|_{\bmR(t)}^2 dt \le \bmk(t_0).
\end{equation*}
To avoid dealing with two-types of constraints the authors proposed to augment the state-space with a new state $k(t)$ as follows:
\begin{equation*}
\begin{aligned}
    \dot \bmx &= \bmA(t) \bmx(t) + \bmB(t) \bmu(t) + \bmC(t) \bmv(t),\\
    \dot \bmk &= -\|\bmu(t)\|_{\bmR(t)}^2.
\end{aligned}
\end{equation*}
Now the integral constraint can be enforced as an end point constraint $k(t_1) \ge 0$. We will not go into further detail about this work but mention that all the matrices $\bmA$, $\bmB$, $\bmC$, $\bmR$ as well as set-valued maps $\cU(t)$, $\cV(t)$ are assumed to be known. In our case, we assume unknown dynamics.

\subsection{By Chow et al} \label{app:state_augmentation_chow}
\cite{chow2017risk} considered safe RL with CVaR constraints and addressed the following optimization problem:
\begin{align}
        \min\limits_{\pi, \nu}~&\mathbb{E} \sum\limits_{t=0}^T  \gamma_c^t c(\bms_t, \bma_t), \label{si_prob:cvar_constrained_mdp}\\
        \text{s.t.: }& \bma_t \sim \pi(\cdot| \bms_t, \bms_{t-1}, \bma_{t-1}, \dots, , \bms_0, \bma_0) \notag\\
        &\nu + \frac{1}{1 - \alpha}\mathbb{E}~\relu \left(\sum\limits_{t=0}^T \gamma_{l}^t l(\bms_t, \bma_t) - \nu\right)  \le d \notag,
\end{align} 
where $T$ is the control horizon. In this formulation, one also needs to consider the target state $\bms_{\rm Tar}$, which signifies the end of the episode. 

For their state augmentation approach~\cite{chow2017risk} proposed the following augmented MDP:
\begin{equation}
\begin{aligned}
     \bms_{t+1} &\sim p(\cdot | \bms_t , \bma_t), \bms_0 \sim {\cal S}_0, \\
     \bmx_{t+1} &= (\bmx_t -  l(\bms_t, \bma_t))/ \gamma_l, \bmx_0 = \nu,
\end{aligned}\label{eq:chow_augmented_mdp}
\end{equation}
 and the augmented cost
 \begin{equation}
    \widetilde c_\lambda(\bms_t, \bmx_t, \bma_t) = \begin{cases}
     c(\bms_t, \bma_t) & \bms_t \ne \bms_{\rm Tar} \\
     \dfrac{\lambda \relu(-\bmx_t)}{1-\alpha} & \text{otherwise}
     \end{cases}
 \end{equation}
The optimization problem that they considered was as follows:
\begin{equation}\label{eq:chow_objective}
\begin{aligned}
    \min\limits_{\pi, \nu} \max_{\lambda \ge 0}~&\mathbb{E} \sum\limits_{t=0}^T  \gamma_c^t \widetilde c_\lambda(\bms_t, \bmx_t, \bma_t),  \\
    & \bma_t \sim \pi(\cdot| \bms_t, \bmx_t).
\end{aligned}
\end{equation}
This idea of state augmentation was introduced by~\cite{ott2010markov, bauerle2011markov} who showed that the optimal policy of a CVaR optimization problem (unconstrained) must depend on the current state $\bms_t$ as well as the history of the accumulated cost $\bmx_t$. In our understanding, however, the representation of the optimal policy for the CVaR-constrained optimization was not discussed by~\cite{ott2010markov, bauerle2011markov}. Hence there may still be an open question of the validity of the Bellman equation for Equation~\ref{si_prob:cvar_constrained_mdp} and the representation of the optimal policy. We stress that \cite{chow2017risk} provided many other theoretical results justifying their approach to CVaR constrained reinforcement learning, but perhaps some gaps remain. 

Note that the finite horizon is only a technical difference between our formulation and the formulation by~\cite{chow2017risk} and is of no consequence. The difference is the initial value of the augmented state $\bmx_0$, which is equal to $\nu$ instead of $d$ as it is in our case. Although this seems to be a subtle difference, it allows for many features such as plug-n-play methods, generalization across safety budgets, learning safe policy for all safety budgets $d$ in some interval $[d_{\rm lower}, d_{\rm upper}]$. 

\subsection{By Calvo-Fullana et al}\label{app:state_augmentation_calvo}
The authors consider the following problem 
\begin{equation}
\begin{aligned}
        \max\limits_{\pi}~&\lim\limits_{T\rightarrow\infty}\mathbb{E}_{\bms,\bma\sim\pi} \sum\limits_{t=0}^T r(\bms_t, \bma_t),\\
        \text{s.t.: }&\lim\limits_{T\rightarrow\infty}\mathbb{E}_{\bms,\bma\sim\pi} \sum\limits_{t=0}^T r_i(\bms_t, \bma_t) \ge d_i,\\
        & \bma_t \sim \pi(\cdot| \bms_t, \bms_{t-1}, \bma_{t-1}, \dots, , \bms_0, \bma_0),
\end{aligned} \label{si_prob:lim_mean}
\end{equation}
and denoted the objective as 
\begin{equation*}
    V_i(\pi) \triangleq \lim\limits_{T\rightarrow\infty} \frac{1}{T} \E_{\bms,\bma\sim\pi} \left[ \sum\limits_{t=0}^T r_i(\bms_t, \bma_t) \right],
\end{equation*}
and the optimal cost (sic) as $V_0(\pi^\ast)$. Then the authors defined the Lagrangian for the primal-dual solution:
\begin{equation*}
    \cL(\pi, \lambda) = V_0(\pi) + \sum\limits_{i =1}^m \lambda_i \left(V_i(\pi) - c_i\right).
\end{equation*}

The solution was proposed by computing $\argmax\limits_{\pi} \cL(\pi, \lambda)$, where the optimal Lagrangian multipliers need to be optimized over. \cite{calvo2021state} propose to update the multipliers as follows:
\begin{equation*}
    \lambda_{i, k+1} = \left[\lambda_{i,k} - \frac{\eta_\lambda}{T_0} \sum\limits_{t=k T_0}^{(k+1) T_0 - 1} (r_i(\bms_t, \bma_t) - d_i) \right],
\end{equation*}
where $\eta_\lambda$ is the step size, $T_0$ is the epoch duration, $k$ is the iteration index. The policy $\pi$ in this formulation depends on the state $\bms_t$ and Lagrangian multipliers $\lambda_i$. Using this idea the authors show that there exists a policy that allows constraint satisfaction with probability one:
\begin{equation*}
    \lim\limits_{T\rightarrow\infty}\mathbb{E}_{\bms,\bma\sim\pi} \sum\limits_{t=0}^T r_i(\bms_t, \bma_t) \ge d_i, \forall i \text{  a. s.}
\end{equation*}
similarly to our case. As we discuss, however, it is not necessary to use the Lagrangian formulation and such complicated constructions to arrive at a similar conclusion. Furthermore, it is not clear if the algorithm trains a policy satisfying the constraint almost surely.

\section{Implementation Details} \label{app:detailed_implementation_details}
Our implementation is freely available online~\cite{sootla_saute_2022_git}. 

\subsection{Saut\'e RL}
The main benefit of our approach to safe RL is the ability to extend it to \emph{any critic-based RL algorithm}. This is because we do not need to change the algorithm itself (besides some cosmetic changes), but create a wrapper around the environment. The implementation is quite straightforward and the only ``trick'' we had to resort to is normalizing the safety state by dividing with the safety budget:
\begin{equation*}
    \bmz_{t+1} = (\bmz_t - l(\bms_t, \bma_t) / d) / \gamma_l, \bmz_0 = 1.
\end{equation*}
Hence the variable $\bmz_t$ is always between zero and one. 

{\scriptsize
\begin{minted}{Python}
def safety_step(self, cost: np.ndarray) -> np.ndarray:
    """
    Update the normalized safety state
    """        
    # subtract the normalized cost
    self._safe_state -= cost / self.safe_budget
    # normalize by the discount factor
    self._safe_state /= self.safe_discount_factor
    return self._safe_state
\end{minted}
}

The step function has to be overloaded to augment the safe state and shape the cost.

{\scriptsize
\begin{minted}{Python}
def step(self, action:np.ndarray) -> np.ndarray:
    """
    Step into the environment 
    """
    # get the state of the environment
    next_obs, reward, done, info = super().step(action)        
    # get the safe state
    next_safe_state = self.safety_step(info['cost'])
    # shape the reward
    if next_safe_state <= 0:
        reward = 0 
    # augment the state
    augmented_state = np.hstack([next_obs, 
                                 next_safe_state])
    # save values 
    info['true_reward'] = reward         
    info['next_safe_state'] = next_safe_state
    return augmented_state, reward, done, info
\end{minted}
}

Note that in this implementation, we assume that the minimum reward is zero and the maximum reward is nonnegative. In some environments, we set a different minimum reward (i.e., we shape $\widetilde c_n$ with a different value $n$). Finally, resetting the environment requires only state augmentation.

{\scriptsize
\begin{minted}{Python}
def reset(self) -> np.ndarray:
    """
    Reset the environment
    """
    # get the state of the environment
    state = super().reset()
    # reset the safe state 
    self._safe_state = 1.0
    # augment the state
    augmented_state = np.hstack([state, 
                                 self._safe_state])
    return augmented_state
\end{minted}
}

We used safety starter agents~\cite{raybenchmarking} (Tensorflow == 1.13.1) as the core implementation for model-free methods and their Lagrangian versions (PPO, TRPO, SAC, CPO). We have tested some environments using stable baselines library~\cite{stable-baselines3} (PyTorch >= 1.8.1) and did not find drastic performance differences. Our choice of safety starter agents is guided only by the implementation of the Lagrangian version of PPO, TRPO, and SAC as well as CPO, which we modify by saut\'eing PPO, TPRO, and SAC.

\subsection{CVaR Lagrangian} \label{app:cvar}

Our implementation of CVaR Lagrangian leverages the implementation by~\cite{cowen2020samba} while implementing the algorithm within the safety starter agents~\cite{raybenchmarking}. Our implementation is quite straightforward as all we need to do is to replace the empirical estimator of the mean with the empirical $\cvar$ estimator. First, let us recall an alternative definition of $\cvar$ for a random variable $X$ with a probability density function $p(X)$ and cumulative probability distribution $F_X(x) = \P(X < x)= \int_{X < x} p(X) dX$ given by~\cite{rockafellar2000optimization}:
\begin{equation*}
\begin{aligned}
\var_\alpha(X) &= \min\{x\in \R | F_X(x) \ge \alpha \}, \\
\cvar_\alpha (X) &= (1-\alpha)^{-1} \int\limits_{F_X(x) \ge \var_\alpha(X)} F_X(x) p(X) dX.
\end{aligned}
\end{equation*}
This definition is equivalent to ours as shown in~\cite{rockafellar2000optimization}. Note that this definition implies that $\cvar_\alpha(X)$ is the mean of the tail of the distribution defined by quantile $\alpha$. Therefore, instead of using the definition
\begin{equation*}
\cvar_\alpha(X) = \min\limits_{\nu \in \R} \left(\eta + \frac{1}{1-\alpha} \mathbb{E}~\relu( X - \nu) \right)    
\end{equation*}
we can simply compute an empirical estimate of the quantile and then compute the empirical estimate of the mean of the tail. Following~\cite{tamar2015optimizing} our algorithm for the empirical $\cvar$ estimate is as follows.

\begin{enumerate}
    \item Let $x_1$, $\dots$, $x_n$ be drawn from $X$ and let $\alpha \in (0, 1)$;
    \item Determine the quantile $q = \lceil n (1- \alpha) \rceil$;
    \item Sort the samples to get $x_{i_1}$, $\dots$, $x_{i_n}$ with $x_{i_1} < \dots < x_{i_n}$ and $i_1$, $\dots$, $i_n$ being a permutation of $1$, $\dots$, $n$;
    \item Compute the empirical estimate as follows:
    \begin{equation*}
        \widehat{\cvar_\alpha}(X) = \frac{1}{n-q}\sum_{j = q}^n x_{i_j}.
    \end{equation*}
\end{enumerate}

Now this estimate can simply replace the empirical mean estimate for the computation of advantage for the safety cost. No other changes to the algorithms are required.

\section{Environments}\label{app:environments}

\begin{figure}
     \centering
     \begin{subfigure}[b]{0.6\columnwidth}
     \centering
     \includegraphics[width=0.99\textwidth]{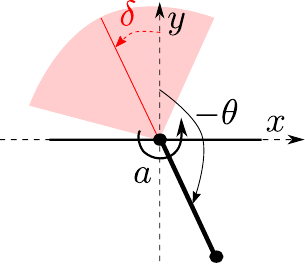} 
     \caption{Safe pendulum swing-up.}
     \label{fig:safe_pendulum}
     \end{subfigure}
     \begin{subfigure}[b]{0.6\columnwidth}     
     \centering
     \includegraphics[width=0.99\textwidth]{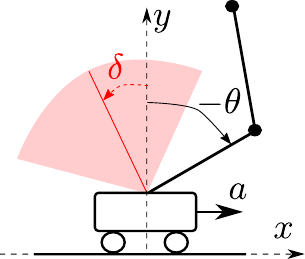} 
     \caption{Safe double pendulum}
     \label{fig:safe_double_pendulum}
     \end{subfigure}
     \vspace{0.1mm}
     \begin{subfigure}[b]{0.6\columnwidth}     
     \centering
     \includegraphics[width=0.99\textwidth]{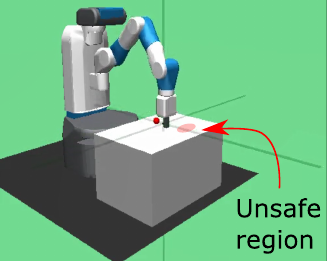} 
     \caption{Safe reacher}
     \label{fig:safe_fetch_reacher}
     \end{subfigure}
     \begin{subfigure}[b]{0.58\columnwidth}     
     \centering
     \includegraphics[width=0.99\textwidth]{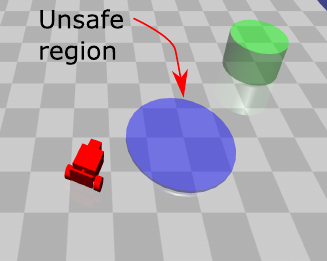} 
     \caption{Safety gym}
     \label{fig:static_env_gym}
     \end{subfigure}
    \caption{Panels a and b: safe pendulum environments. In both cases, $\theta$ - is the angle from the upright position, $a$ is the action, $\delta$ - is the unsafe pendulum angle, the safety cost is the distance toward the unsafe pendulum angle, which is incurred only in the red area. Panel e: safe reacher:  the robot needs to avoid the unsafe region. Panel d: a schematic depiction of the safety gym environment: robot needs to reach the goal while avoiding the unsafe region.}
    \label{fig:envs}
\end{figure}
\paragraph{Pendulum Swing-up.} We take the single pendulum swing-up from the classic control library in the Open AI Gym~\cite{brockman2016openai}. However, we define the instantaneous task cost following~\cite{cowen2020samba} as follows:
\begin{equation*}
    c(\bms, \bma) = 1 - \frac{\bmtheta ^2 + 0.1 \dot \bmtheta^2 + 0.001 \bma^2}{\pi^2 + 6.404},
\end{equation*}
which takes values between zero and one, since $\bma \in [-2, 2]$, $\bms[0] = \bmtheta\in [-\pi, \pi]$ $\bms[1] = \dot \bmtheta \in [-8,8]$. We define the instantaneous safety cost following~\cite{cowen2020samba}:
\begin{equation*}
    l = \begin{cases}
    1 - \dfrac{|\theta - \delta|}{50} & \text{ if } -25 \le \theta \le 75, \\
    0 & \text{otherwise,}
    \end{cases}
\end{equation*}
where $\theta$ is the angle (in degrees) of the pole deviation from the upright position. The cost is designed to create a trade-off between swinging up the pendulum and keeping away from the angle $\delta = 25^o$. This environment has three states (cosine and sine of $\theta$, as well as angular velocity $\dot \theta$) and one action. See the depiction in Figure~\ref{fig:safe_pendulum}.

\paragraph{Double Pendulum.} We take the double pendulum stabilization implementation by~\cite{todorov2012mujoco} using the Open AI Gym~\cite{brockman2016openai} interface (the environment \mintinline{python}{InvertedDoublePendulumEnv} from \mintinline{python}{gym.envs.mujoco}). We modify the environment by setting the maximum episode length to $200$ and divide the instantaneous reward by $10$. We used $n=200$ to get $\widetilde c_n$.
 
We define safety similarly to the pendulum swing-up case, i.e., we use the same instantaneous cost with $\theta$ is angle (in degrees) of the first pole deviation from the upright position and define the cost as follows:
\begin{equation*}
    l = \begin{cases}
    1 - \dfrac{|\theta - \delta|}{50} & \text{ if } -25 \le \theta \le 75, \\
    0 & \text{otherwise.}
    \end{cases}
\end{equation*}
This environment has eleven states and one action. See the depiction in Figure~\ref{fig:safe_double_pendulum}. 

\paragraph{Reacher.} We take the reacher implementation by~\cite{todorov2012mujoco} using the Open AI Gym~\cite{brockman2016openai} interface (\mintinline{python}{Reacher}). We add the following safety cost for this environment
\begin{equation*}
    l = \begin{cases}
    100 - 50 \cdot |\bmx - \bmx_{\rm target}| & \text{if } |\bmx - \bmx_{\rm target}| \le 0.5, \\
    0 & \text{otherwise,}
    \end{cases}
\end{equation*}
where $\bmx_{\rm target}$ is the position of the target (set to $\begin{pmatrix}
1.0 & 1.0 & 0.01 
\end{pmatrix}$) and $\bmx$ is the position of the arm in Cartesian coordinates.  The environment is schematically depicted in Figure~\ref{fig:safe_fetch_reacher}.
Overall the system has eleven states and two actions. 

\begin{table}[t]
    \centering
    \caption{
    	Default hyperparameters for TRPO, PPO, CPO
    } \label{table:parameters_PG}
    {
    \begin{tabular}{@{}llccc@{}}
    \toprule
    \multicolumn{2}{l}{Name} & Value \\ \midrule
    \multicolumn{1}{c}{\multirow{8}{*}{\rotatebox{90}{Common parameters}}} 
        & Network architecture         & [64,64]            \\
        & Activation                   & tahn                \\
        & Value function learning rate & 1e-3                \\
        & Task Discount Factor         & 0.99                \\
        & Lambda                       & 0.97                \\
        & N samples per epochs         & 1000                \\
        & N gradient steps             & 80                  \\
        & Target KL                    & 0.01                \\
    \midrule
    \multicolumn{1}{c}{\multirow{4}{*}{\rotatebox{90}{Safety}}}
        & Penalty learning rate        & 5e-2                \\
        & Safety Discount Factor       & 0.99                \\        
        & Safety Lambda                & 0.97                \\        
        & Initial penalty              & 1                   \\        
    \midrule
    \multicolumn{1}{c}{\multirow{4}{*}{\rotatebox{90}{PPO}}}
        & Clip ratio                  & 0.2                  \\
        & Policy learning rate        & 3e-4                 \\
        & Policy iterations           & 80                   \\
        & KL margin                   & 1.2                  \\
    \midrule
    \multicolumn{1}{c}{\multirow{4}{*}{\rotatebox{90}{TRPO/CPO}}}
        & Damping Coefficient         & 0.1                  \\
        & Backtrack Coefficient       & 0.8                  \\
        & Backtrack iterations        & 10                   \\
        & Learning Margin             & False                \\
    \bottomrule
    \end{tabular}
    }
\end{table}
\paragraph{Safety Gym.} We take the Static environment from~\cite{yang2021wcsac}, which is modifications of the safety gym environments~\cite{raybenchmarking}. The unsafe region is placed near the goal and the robot is placed randomly, after the goal is reached the episode ends. The safety cost is incurred at every time step spent in the blue area. This environment is schematically depicted in Figure~\ref{fig:static_env_gym}. We consider two robots: ``point'' with $46$ states and $2$ actions, and ``car'' with $56$ states and $2$ actions. Note that our variant of the safety gym environment has deterministic constraints, not randomly placed constraints as is common in other safety gym environments. We did so because our algorithm aims to deliver almost surely constraints and hence we did not want to contaminate our experiments with unsolvable problems.

\section{Experiment details} \label{app:default_hyperparameters}
We take the default parameters presented in Tables~\ref{table:parameters_PG} and \ref{table:parameters_SAC}. For all modifications of TRPO/PPO  and SAC, the default parameters and the code base is the same, which makes the direct comparisons fairer. Note that these parameters are used in the safety starter agents implementation of these algorithms~\cite{raybenchmarking}.

\begin{table}[t]
    \centering
    \caption{
    	Default hyperparameters for SAC
    } \label{table:parameters_SAC}
    \begin{tabular}{ll}
    \toprule
        Name & Value \\ \midrule
        Network architecture         & [256,256]           \\
        Activation                   & ReLU               \\
        Value function learning rate & 5e-4                \\
        Policy learning rate         & 5e-4                \\
        $\alpha$ learning rate       & 5e-2                \\
        Batch size                   & 1024                \\
        Task Discount Factor         & 0.99                \\
        N samples per epochs         & 200                \\
        Training frequency           & 1                   \\
        Target entropy               & -$|\cA|$        \\
        $\tau$                       & 0.005               \\
        Size of the replay buffer    & 1e6                 \\
        N start updates              & 1e3                 \\
        Penalty learning rate        & 5e-2                \\
        Safety Discount Factor       & 0.99                \\        
    \bottomrule
    \end{tabular}
\end{table}

\paragraph{Pendulum Swing-up.} We use default parameters. We plot evaluation during training in Figures~\ref{fig:saute_RL_single},~\ref{fig:trpo_single_pendulum},~\ref{fig:ppo_single_pendulum} and~\ref{fig:sac_single_pendulum}. Note that ``saut\'eed'' algorithms achieve a safe almost surely policy after $200$ episodes of training. We plot maximum incurred cost over the episode to evaluate the constraint violation during training. 

\paragraph{Double Pendulum} We use default parameters for Vanilla TRPO, Saut\'e TRPO, and CPO. We run a hyper-parameter search for Lagrangian TRPO by varying the penalty learning rate (5e-3, 1e-2, 5e-2), backtrack iterations ($10$, $15$, $20$), value function learning rate (1e-4, 1e-3, 5e-3) and steps per epoch ($4000$, $10000$, $20000$). However, we did not find any parameter setting that performs significantly better than the default one. We plot evaluation during training in Figure~\ref{fig:trpo_double_pendulum}. We plot the maximum incurred cost over the episode to evaluate the constraint violation during training.  We also present results for ablation on the cost function in Table~\ref{tab:reward_shaping}. 

\paragraph{Reacher} We use default parameters for Vanilla TRPO, Saut\'e TRPO, and CPO. We run a hyper-parameter search for Lagrangian TRPO by varying the penalty learning rate (5e-3, 1e-2, 5e-2), backtrack iterations ($10$, $15$, $20$), value function learning rate (1e-4, 1e-3, 5e-3) and steps per epoch ($4000$, $10000$, $20000$). We finally determined the following customized parameters after comparison: penalty learning rate (3e-2), backtracking iterations ($15$), value function learning rate (1e-2), steps per epoch ($4000$), and number of epochs($1000$). Using this parameter setting makes the average cost fluctuates slightly around the safety budget, unlike the average cost of CPO which converges to the safety budget.

\paragraph{Safety Gym} We use default parameters for Vanilla TRPO, Lagrangian TRPO, and CPO in both Point Goal and Car Goal. We run a hyper-parameter search for Saut\'e TRPO by varying the penalty learning rate (5e-3, 1e-2, 5e-2), backtrack iterations ($10$, $15$, $20$), value function learning rate (1e-4, 1e-3, 5e-3) and steps per epoch ($4000$, $6000$, $10000$, $15000$). We finally determined the following customized parameters after comparison: penalty learning rate (3e-2), backtracking iterations ($15$), value function learning rate (5e-3), steps per epoch ($10000$), and the number of epochs($2000$). Using this parameter setting keeps most of the costs below the safety budget while most of the task costs are similar to the task costs of other algorithms (see percentiles and medians in the box plots).

\begin{figure}
     \centering
     \begin{subfigure}[b]{0.7\columnwidth}
     \centering
     \includegraphics[width=0.99\textwidth]{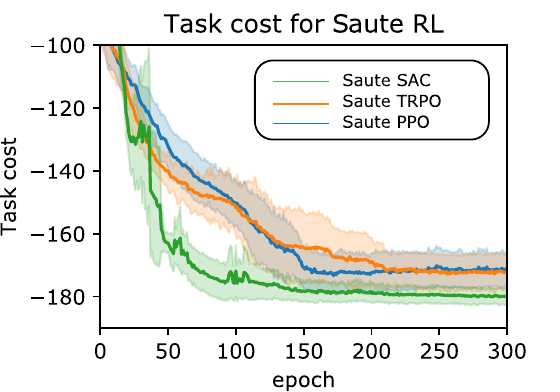} 
     \caption{Mean task costs}
     \label{fig:mean_task_cost_saute_single}
     \end{subfigure} 
     \begin{subfigure}[b]{0.7\columnwidth}     
     \centering
     \includegraphics[width=0.99\textwidth]{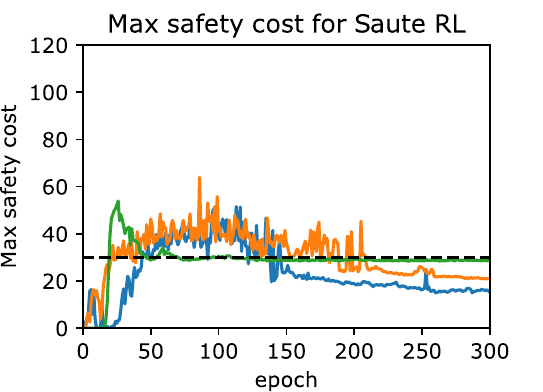}
     \caption{Max safety costs}
     \label{fig:max_safety_cost_saute_single}
     \end{subfigure}
      \begin{subfigure}[b]{0.7\columnwidth}
     \centering
     \includegraphics[width=0.99\textwidth]{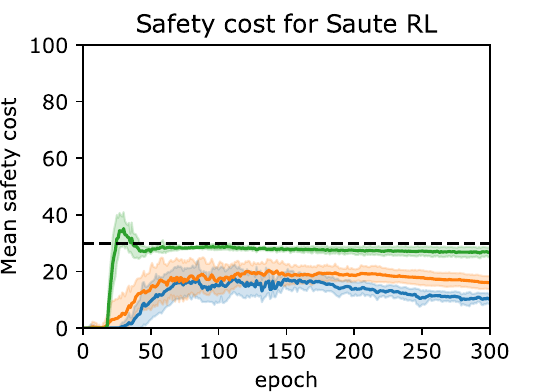} 
     \caption{Mean safety costs}
     \label{fig:mean_safety_cost_saute_single}
     \end{subfigure} 
     \begin{subfigure}[b]{0.7\columnwidth}     
     \centering
     \includegraphics[width=0.99\textwidth]{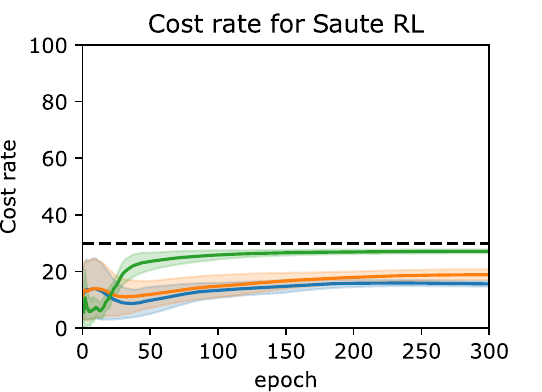}
     \caption{Cost rates}
     \label{fig:cost_rate_saute_single}
     \end{subfigure}
    \caption{Evaluation results of Saut\'e PPO, Saut\'e TRPO and Saut\'e SAC on the pendulum swing-up task over 5 different seeds with $100$ trajectories for every seed. Average task cost are depicted in Panel a (shaded areas are the standard deviation over all runs), maximum incurred safety costs are depicted in Panel b, average incurred safety costs are depicted in Panel c (shaded areas are the standard deviation over all runs), and cost rates are depicted in Panel d (shaded areas are the standard deviation over different seeds). The black-dotted line is the safety budget used for training.}
    \label{fig:saute_RL_single}
\end{figure}

\begin{figure}
     \centering
     \begin{subfigure}[b]{0.69\columnwidth}
     \centering
     \includegraphics[width=0.99\textwidth]{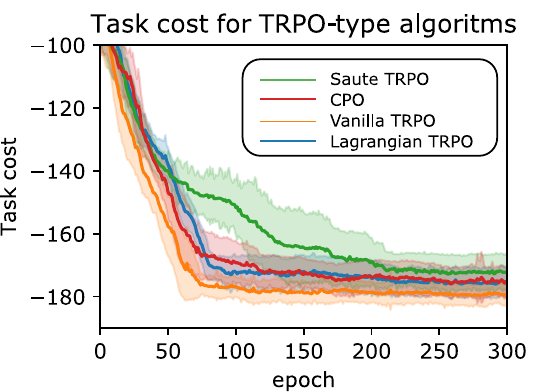}
     \caption{Mean task costs}
     \label{fig:mean_task_cost_trpo_single}
     \end{subfigure}
     \begin{subfigure}[b]{0.69\columnwidth}
     \centering
     \includegraphics[width=0.99\textwidth]{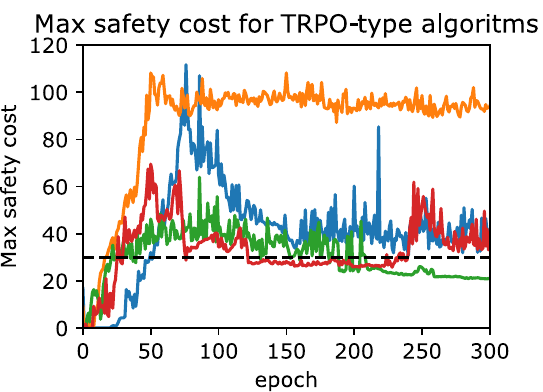}
     \caption{Max safety costs}
     \label{fig:max_safety_cost_trpo_single}
     \end{subfigure}
     \begin{subfigure}[b]{0.69\columnwidth}     
     \centering
     \includegraphics[width=0.99\textwidth]{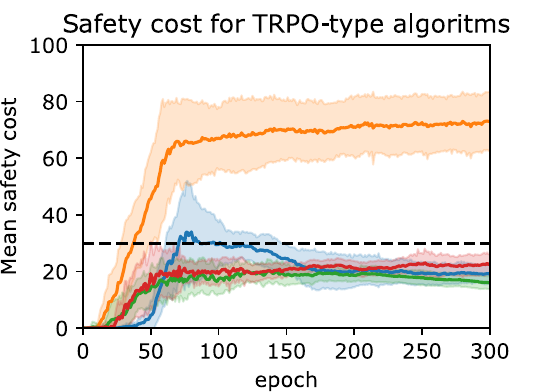}
     \caption{Mean safety costs}
     \label{fig:maean_safety_cost_trpo_single}
     \end{subfigure}
     \begin{subfigure}[b]{0.69\columnwidth}     
     \centering
     \includegraphics[width=0.99\textwidth]{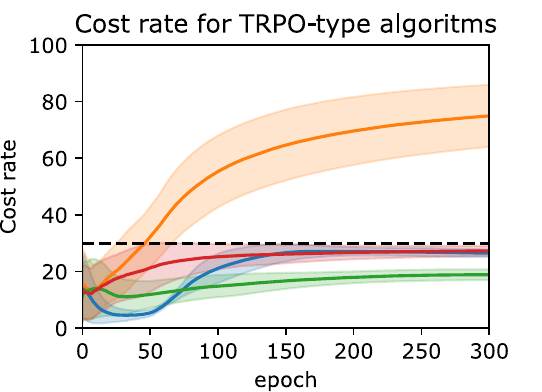}
     \caption{Cost rate}
     \label{fig:cost_rate_trpo_single}
     \end{subfigure}     
    \caption{Evaluation results of Vanilla TRPO, Saut\'e TRPO, Lagrangian TRPO and CPO on the pendulum swing-up task over 5 different seeds with $100$ trajectories for every seed. Average task cost are depicted in Panel a (shaded areas are the standard deviation over all runs), maximum incurred safety costs are depicted in Panel b, average incurred safety costs are depicted in Panel c (shaded areas are the standard deviation over all runs), and cost rates are depicted in Panel d (shaded areas are the standard deviation over different seeds). The black-dotted line is the safety budget used for training. }     \label{fig:trpo_single_pendulum}
\end{figure}

\begin{figure}
     \centering
     \begin{subfigure}[b]{0.69\columnwidth}
     \centering
     \includegraphics[width=0.99\textwidth]{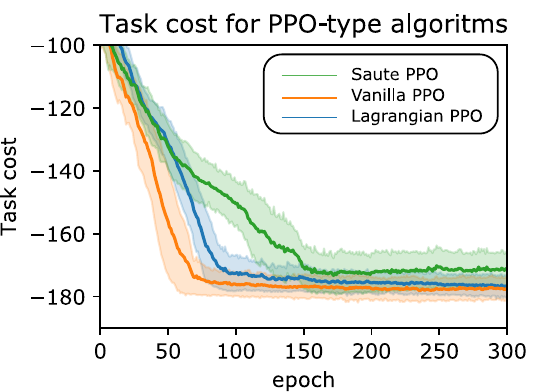}
     \caption{Mean task costs}
     \label{fig:mean_task_cost_ppo_single}
     \end{subfigure}
     \begin{subfigure}[b]{0.69\columnwidth}
     \centering
     \includegraphics[width=0.99\textwidth]{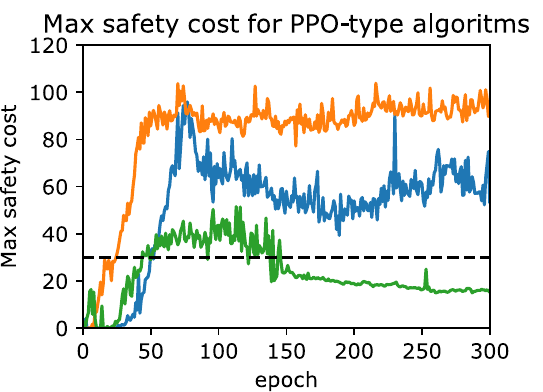}
     \caption{Max safety costs}
     \label{fig:max_safety_cost_ppo_single}
     \end{subfigure}
     \begin{subfigure}[b]{0.69\columnwidth}     
     \centering
     \includegraphics[width=0.99\textwidth]{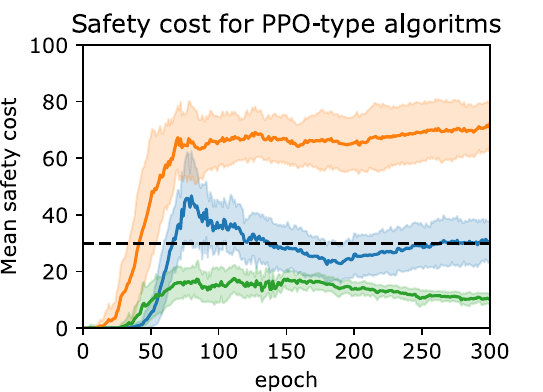}
     \caption{Mean safety costs}
     \label{fig:mean_safety_cost_ppo_single}
     \end{subfigure}
     \begin{subfigure}[b]{0.69\columnwidth}     
     \centering
     \includegraphics[width=0.99\textwidth]{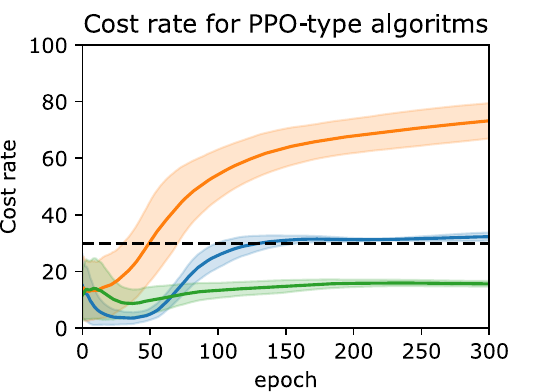}
     \caption{Cost rate}
     \label{fig:cost_rate_ppo_single}
     \end{subfigure}     
    \caption{Evaluation results of Vanilla PPO, Saut\'e PPO and Lagrangian PPO on the pendulum swing-up task over 5 different seeds with $100$ trajectories for every seed. Average task cost are depicted in Panel a (shaded areas are the standard deviation over all runs), maximum incurred safety costs are depicted in Panel b, average incurred safety costs are depicted in Panel c (shaded areas are the standard deviation over all runs), and cost rates are depicted in Panel d (shaded areas are the standard deviation over different seeds). The black-dotted line is the safety budget used for training. }    \label{fig:ppo_single_pendulum}
\end{figure}

\begin{figure}
     \centering
     \begin{subfigure}[b]{0.69\columnwidth}
     \centering
     \includegraphics[width=0.99\textwidth]{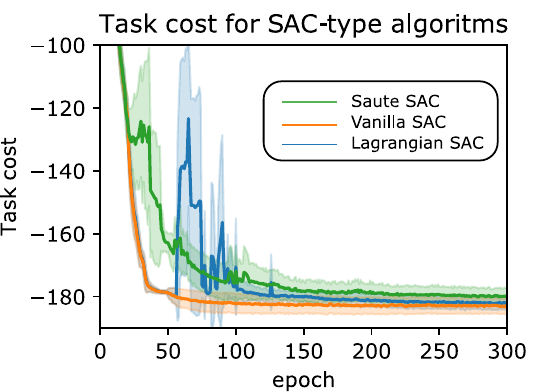}
     \caption{Mean task costs}
     \label{fig:mean_task_cost_sac_single}
     \end{subfigure}
     \begin{subfigure}[b]{0.69\columnwidth}
     \centering
     \includegraphics[width=0.99\textwidth]{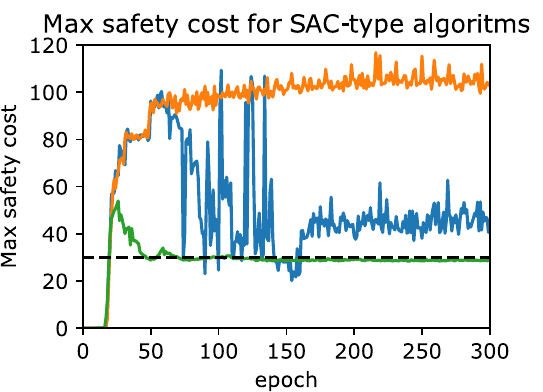}
     \caption{Max safety costs}
     \label{fig:max_safety_cost_sac_single}
     \end{subfigure}
     \begin{subfigure}[b]{0.69\columnwidth}     
     \centering
     \includegraphics[width=0.99\textwidth]{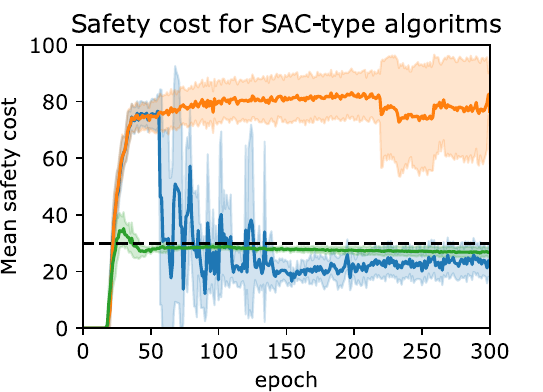}
     \caption{Mean safety costs}
     \label{fig:mean_safety_cost_sac_single}
     \end{subfigure}
     \begin{subfigure}[b]{0.69\columnwidth}     
     \centering
     \includegraphics[width=0.99\textwidth]{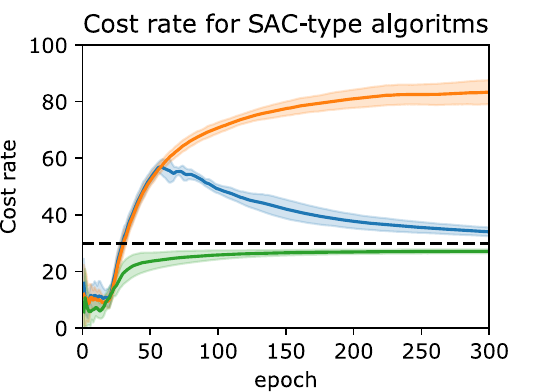}
     \caption{Cost rate}
     \label{fig:cost_rate_sac_single}
     \end{subfigure}     
    \caption{Evaluation results of Vanilla SAC, Saut\'e SAC and Lagrangian SAC on the pendulum swing-up task over 5 different seeds with $100$ trajectories for every seed. Average task cost are depicted in Panel a (shaded areas are the standard deviation over all runs), maximum incurred safety costs are depicted in Panel b, average incurred safety costs are depicted in Panel c (shaded areas are the standard deviation over all runs), and cost rates are depicted in Panel d (shaded areas are the standard deviation over different seeds). The black-dotted line is the safety budget used for training. }     \label{fig:sac_single_pendulum}
\end{figure}

\begin{figure}
     \centering
     \begin{subfigure}[b]{0.33\textwidth}
     \centering
     \includegraphics[width=0.99\textwidth]{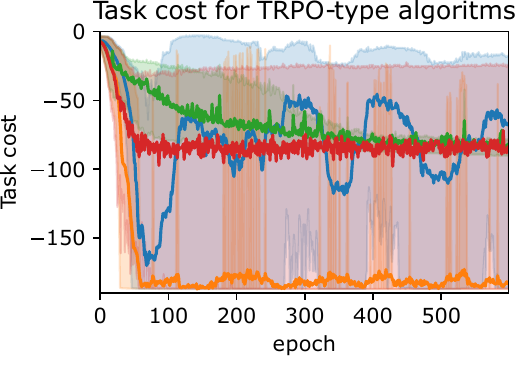}
     \caption{Mean task costs}
     \label{fig:mean_task_cost_trpo_double_pendulum}
     \end{subfigure}
     \begin{subfigure}[b]{0.33\textwidth}
     \centering
     \includegraphics[width=0.99\textwidth]{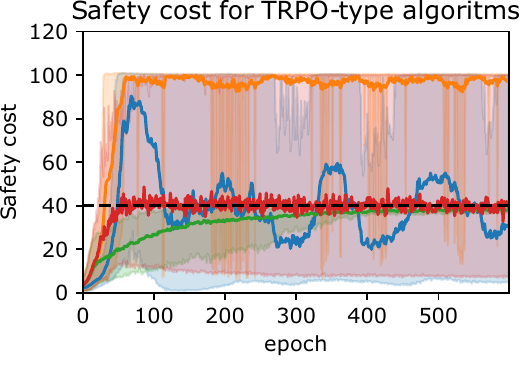}
     \caption{Mean safety costs}
     \label{fig:mean_safety_cost_trpo_double_pendulum}
     \end{subfigure}     
     \begin{subfigure}[b]{0.33\textwidth}
     \centering
     \includegraphics[width=0.99\textwidth]{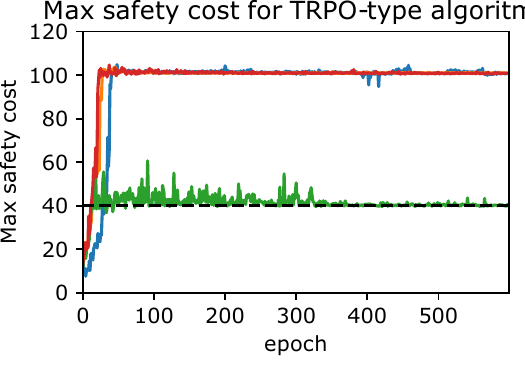}
     \caption{Max safety costs}
     \label{fig:max_safety_cost_trpo_double_pendulum}
     \end{subfigure}
    \caption{Evaluation results  on the double pendulum environment. Vanilla TRPO, Saut\'e TRPO, Lagrangian TRPO and CPO are evaluated on 5 different seeds with $100$ trajectories for every seed. Average task costs are depicted in Panel a (shaded areas are $80\%$ percentile intervals), Average task costs are depicted in Panel b (shaded areas are $80\%$ percentile intervals), maximum incurred safety costs are depicted in Panel c. The black-dotted line is the safety budget ($40$) used for training}
    \label{fig:trpo_double_pendulum}
\end{figure}

\begin{table}\centering
\caption{Task (bold burgundy) and safety (italic blue) costs for Saut\'e TRPO and various reshaped costs $\widetilde c_n$. The first value in the bracket is $5\%$ percent quantile, the second is the mean, and the the third is $95\%$ percent quantile.}\label{tab:reward_shaping}
\resizebox{\columnwidth}{!}{
\begin{tabular}{rcc}
\toprule
n & Task costs & Safety costs \\
\midrule
0    &    {\color{cb-burgundy}$\bf [-61.47, -74.87, -86.67]$ }& {\color{cb-blue-sky} $\it [32.49, 37.56, 39.68]$ }\\
10   &    {\color{cb-burgundy}$\bf [-64.34, -71.68, -83.01]$ }& {\color{cb-blue-sky} $\it [36.84, 38.38, 39.47]$} \\
100  &    {\color{cb-burgundy}$\bf [-72.65, -79.29, -91.18]$} & {\color{cb-blue-sky} $\it [35.91, 37.71, 39.69]$ }\\
1000 &    {\color{cb-burgundy}$\bf [-80.24, -83.95, -89.16]$} & {\color{cb-blue-sky} $\it [37.05, 38.15, 39.40]$} \\
10000  &  {\color{cb-burgundy}$\bf [-67.24, -78.95, -86.73]$} & {\color{cb-blue-sky} $\it[36.94, 38.09, 39.50]$} \\
100000 &  {\color{cb-burgundy}$\bf [-67.19, -76.42, -85.81]$} & {\color{cb-blue-sky} $\it[36.22, 37.92, 39.45]$} \\
\bottomrule
\end{tabular}
}
\end{table}

\section{Further experiments} \label{app:new_experiments}
\subsection{SAC with narrower networks}
To compare directly PPO, TRPO, and SAC we performed another experiment while setting policy network architecture to $[64,64]$ for SAC. We report the results in Figure~\ref{fig:sac_single_pendulum_2}. It is clear that Saut\'e SAC is still preferable to Lagrangian SAC in terms of safety almost surely, but increasing the width of the neural networks for its actors and critics improves the performance at the test time. In this work, we have not experimented more with SAC algorithms since TRPO and PPO-based are the basis for the state-of-the-art approaches.

\begin{figure}
     \centering
     \begin{subfigure}[b]{0.69\columnwidth}
     \centering
     \includegraphics[width=0.99\textwidth]{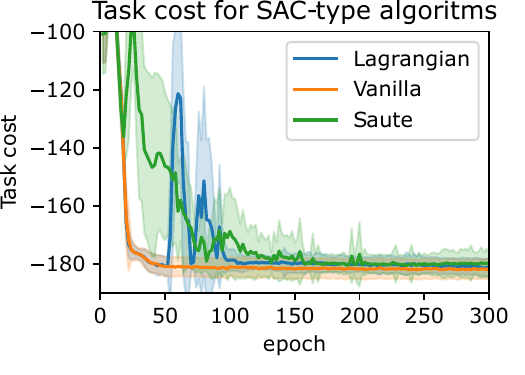}
     \caption{Mean task costs}
     \label{fig:mean_task_cost_sac_single_2}
     \end{subfigure}
     \begin{subfigure}[b]{0.69\columnwidth}
     \centering
     \includegraphics[width=0.99\textwidth]{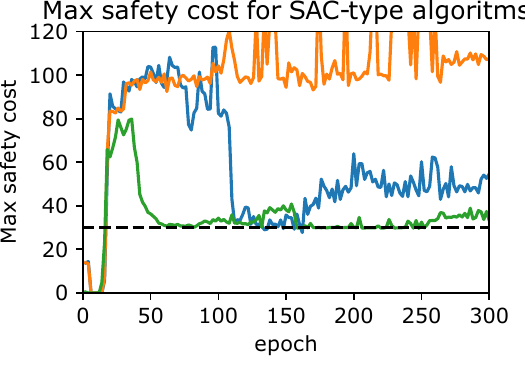}
     \caption{Max safety costs}
     \label{fig:max_safety_cost_sac_single_2}
     \end{subfigure}
     \begin{subfigure}[b]{0.69\columnwidth}     
     \centering
     \includegraphics[width=0.99\textwidth]{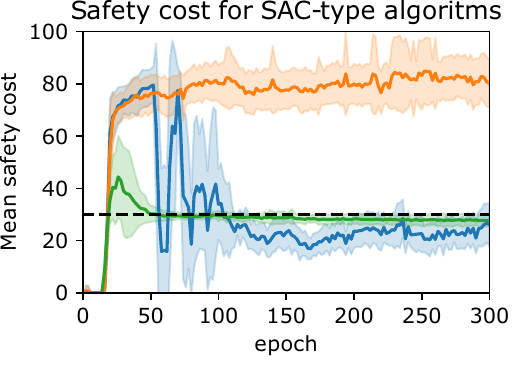}
     \caption{Mean safety costs}
     \label{fig:mean_safety_cost_sac_single_2}
     \end{subfigure}
     \begin{subfigure}[b]{0.69\columnwidth}     
     \centering
     \includegraphics[width=0.99\textwidth]{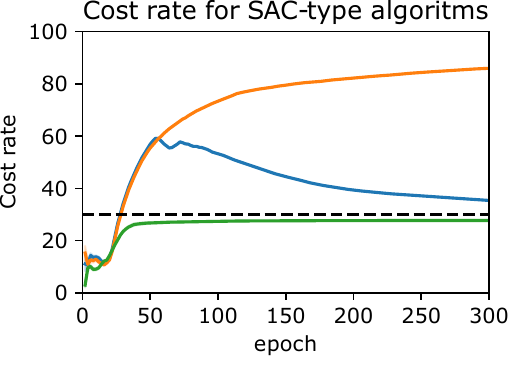}
     \caption{Cost rate}
     \label{fig:cost_rate_sac_single_2}
     \end{subfigure}     
    \caption{Evaluation results of Vanilla SAC, Saut\'e SAC and Lagrangian SAC on the pendulum swing-up task over 5 different seeds with $100$ trajectories for every seed. Here we set the network architecture to $[64,64]$. These results are comparable to the results in Figure~\ref{fig:sac_single_pendulum}, where the network architecture is $[256, 256]$}     \label{fig:sac_single_pendulum_2}
\end{figure}

\subsection{Comparisons to PID-Lagrangian}\label{app:pidl-comparison}
Finally, we compare our method to the PID-Lagrangian approach~\cite{stooke2020responsive}, and we do so on the same tasks and environments. However, instead of saving the policies and validating them separately, we use the training data as it was done in~\cite{stooke2020responsive}. We do so to perform a fair comparison to~\cite{stooke2020responsive}. We make the similar adjustments to the other baselines and present results in Tables~\ref{tab:reacher}, \ref{tab:point}, \ref{tab:car}. We did not observe a significant difference in the performance of PID-Lagrangian and TRPO-Lagrangian even after significant tuning efforts and hence our initial conclusions hold. Also, note that all Lagrangian algorithms did not converge on the Reacher environment. 

\begin{table}
    \centering
\caption{Evaluation of Reacher during training. We report mean $\pm$ standard deviation for task costs; mean, $90\%$ quantile, $99\%$ quantile for safety costs. We havve used $K_i=0.01$, $K_p =0.1$ for PID-L, and $\alpha=0.01$ for CVaR-TRPO. }
\begin{tabular}{lccc}
\toprule
{} &               Task &                 Safety \\
\midrule
Vanilla &  $-60.85 \pm 0.10$ &  $18.15, 19.14, 19.36$ \\
Saut\'e   &  $-61.35 \pm 0.44$ &  $ \textbf{4.51, 8.78, 9.12}$\\
CPO     &  $-60.88 \pm 0.10$ &   $9.95, 11.69, 13.18$ \\
TRPO-L  &  $-61.15 \pm 0.44$ &   $6.44, 12.44, 14.85$ \\
CVaR-L  &  $-60.91 \pm 0.11$ &   $7.81, 13.90, 17.41$ \\
PID-L   &  $-61.17 \pm 0.13$ &   $5.68, 10.02, 10.94$ \\
\bottomrule
\end{tabular}
\label{tab:reacher}
\end{table}

\begin{table}
    \centering
\caption{Evaluation of Point Goal during training. We report mean $\pm$ standard deviation for task costs; mean, $90\%$ quantile, $99\%$ quantile for safety costs. We havve used $K_i=0.01$, $K_p =0.1$ for PID-L, and $\alpha=0.01$ for CVaR-TRPO.}
    \begin{tabular}{lccc}
    \toprule
    {} &             Task &                 Safety \\
    \midrule
    Vanilla &  $2.89 \pm 0.47$ &  $16.18, 52.00, 55.08$ \\
    Saut\'e &  $2.77 \pm 0.59$ &   $\textbf{0.19, 0.00, 0.00}$ \\
    CPO     &  $2.93 \pm 0.49$ &  $17.92, 52.00, 58.04$ \\
    TRPO-L  &  $2.96 \pm 0.45$ &  $19.21, 52.00, 55.05$ \\
    CVaR-L  &  $2.90 \pm 0.48$ &    $0.93, 0.00, 33.02$ \\
    PID-L   &  $2.90 \pm 0.60$ &  $18.24, 52.00, 59.78$ \\
    \bottomrule
    \end{tabular}
\label{tab:point}
\end{table}

\begin{table}
    \centering
\caption{Evaluation of Car Goal during training. We report mean $\pm$ standard deviation for task costs; mean, $90\%$ quantile, $99\%$ quantile for safety costs. We havve used $K_i=0.01$, $K_p =0.1$ for PID-L, and $\alpha=0.01$ for CVaR-TRPO.}
    \begin{tabular}{lccc}
    \toprule
    {} &             Task &                 Safety \\
    \midrule
    Vanilla &  $2.96 \pm 0.46$ &  $19.88, 52.00, 68.02$ \\
    Saut\'e   &  $2.81 \pm 0.45$ &  $\textbf{0.22, 0.00, 0.00}$ \\
    CPO     &  $2.91 \pm 0.46$ &  $16.75, 53.00, 73.00$ \\
    TRPO-L  &  $2.91 \pm 0.47$ &  $16.65, 51.00, 64.07$ \\
    CVaR-L  &  $2.72 \pm 0.64$ &    $1.02, 0.00, 39.00$ \\
    PID-L   &  $2.91 \pm 0.61$ &  $13.72, 48.00, 83.00$ \\
    \bottomrule
    \end{tabular}
\label{tab:car}
\end{table}

\end{document}

%% file: mathdefs.tex
\usepackage{microtype}
\usepackage{booktabs} 

\usepackage{tabularx}
\usepackage{wrapfig}
\usepackage[utf8]{inputenc} 
\usepackage[T1]{fontenc}    
\usepackage{url}            
\usepackage{tgbonum}
\usepackage{amsfonts}       
\usepackage{amsmath}
\usepackage{amssymb}
\usepackage{nicefrac}       
\usepackage{microtype}      
\usepackage{xcolor}
\usepackage{mathtools}
\usepackage{float}
\usepackage{multirow}
\usepackage{pifont}
\usepackage{multicol}
\usepackage{siunitx}
\usepackage{textcomp}
\usepackage{bbm}
\usepackage{bm}
\usepackage{xcolor}
\DeclareMathOperator*{\argmax}{argmax}

\newcommand{\R}{\mathbb{R}}

\renewcommand{\P}{\mathbb{P}} 
\newcommand{\E}{\mathbb{E}}
\newcommand{\K}{\mathbb{K}}

\newcommand{\cA}{{\mathcal A}}

\newcommand{\cL}{{\mathcal L}}
\newcommand{\cM}{{\mathcal M}}

\newcommand{\cP}{{\mathcal P}} 
\newcommand{\cQ}{{\mathcal Q}} 
 
\newcommand{\cS}{{\mathcal S}} 

\newcommand{\cU}{{\mathcal U}}
\newcommand{\cV}{{\mathcal V}}

\newcommand{\cZ}{{\mathcal Z}}

\newcommand{\bma}{{\bm a}}

\newcommand{\bmd}{{\bm d}}

\newcommand{\bmk}{{\bm k}}

\newcommand{\bms}{{\bm s}}

\newcommand{\bmu}{{\bm u}}
\newcommand{\bmv}{{\bm v}}
\newcommand{\bmw}{{\bm w}}

\newcommand{\bmx}{{\bm x}}
\newcommand{\bmy}{{\bm y}}
\newcommand{\bmz}{{\bm z}}

\newcommand{\bmA}{{\bm A}}
\newcommand{\bmB}{{\bm B}}
\newcommand{\bmC}{{\bm C}}

\newcommand{\bmR}{{\bm R}}

\newcommand{\bmtheta}{{\bm \theta}}

\newcommand{\cvar}{\textrm{CVaR}}
\newcommand{\var}{\textrm{VaR}}
\newcommand{\relu}{\textrm{ReLU}}

\newenvironment{proof}{\paragraph{Proof:}}{\hfill$\square$}